%% file: main.tex
\documentclass{article}
\usepackage{arxiv}
\input{includes}

\usepackage[utf8]{inputenc} 
\usepackage[T1]{fontenc}    
\usepackage{booktabs}
\usepackage{amssymb}
\usepackage{amsfonts}       
\usepackage{nicefrac}       
\usepackage{microtype}      
\usepackage{cleveref}       
\usepackage{lipsum}         
\usepackage{graphicx}
\usepackage{natbib}

\title{Auto-weighted Robust Federated Learning with Corrupted Data Sources}

\author{
 Shenghui Li \\
  Uppsala University, Sweden \\
  \texttt{shenghui.li@it.uu.se} \\
   \And
 Edith Ngai \\
  The University of Hong Kong, China \\
  \texttt{chngai@eee.hku.hk} \\
  \And
 Fanhua Ye \\
  University College London, UK \\
  \texttt{fanghua.ye.19@ucl.ac.uk} \\
  \And
 Thiemo Voige \\
  Uppsala University, Sweden \\
  \texttt{thiemo.voigt@it.uu.se} \\
}

\begin{document}
\maketitle

\input{abstract}


\input{body}

\end{document}

%% file: includes.tex
\usepackage{booktabs}       %
\usepackage{amsfonts}       %

\usepackage{microtype}      %
\usepackage{subcaption}
\usepackage{comment}
\usepackage[hidelinks]{hyperref}
\usepackage{url}
\usepackage{color,soul}
\usepackage[switch]{lineno}
\usepackage{amsthm, thmtools}
\usepackage{times}
\usepackage{graphicx} %
\usepackage{color}
\usepackage{wrapfig}
\usepackage{array}
\usepackage{amsmath,bm}
\usepackage{xspace}
\usepackage{fancyhdr}
\usepackage{comment}
\usepackage{color,soul}
\usepackage[inline,shortlabels]{enumitem}
\usepackage{algorithm}
\usepackage{algorithmic}
\usepackage{wrapfig}

\newcolumntype{H}{>{\setbox0=\hbox\bgroup}c<{\egroup}@{}}

\newcommand{\noaistats}[1]{}  %

\definecolor{darkgreen}{rgb}{0,0.4,0.0}
\definecolor{darkblue}{rgb}{0,0.1,0.3}
\definecolor{darkred}{rgb}{0.7,0.0,0.0}

\newcommand{\lbs}{\ensuremath{B}}      %

\newcommand{\lepochs}{\ensuremath{E}}  %

\newcommand{\loss}{\ell}
\newcommand{\risk}{\mathcal{L}}
\newcommand{\uniform}{\overline{\mathcal{U}}}
\newcommand{\hstar}{{h}^{\ast}}
\newcommand{\SUB}[1]{\ENSURE \hspace{-0.15in} \textbf{#1}}
\newcommand{\algfont}[1]{\texttt{#1}}

\newcommand{\fedavg}{\algfont{FedAvg}\xspace}
\newcommand{\FedAvg}{\algfont{FedAvg}\xspace}

\newcommand{\arfl}{\algfont{ARFL}\xspace}
\newcommand{\ARFL}{\algfont{ARFL}\xspace}
\newcommand{\RFA}{\algfont{RFA}\xspace}
\newcommand{\CFL}{\algfont{CFL}\xspace}
\newcommand{\MKrum}{\algfont{MKRUM}\xspace}
\newcommand{\rfa}{\algfont{RFA}\xspace}

\newcommand{\T}{\rule{0pt}{2.2ex}}

\newtheorem{lemma}{Lemma}

\newcommand \expect {\mathbb{E}}

\theoremstyle{definition}
\newtheorem{definition}{Definition}

\DeclareMathOperator*{\argmin}{argmin}
\DeclareMathOperator*{\argmax}{argmax}

%% file: abstract.tex
\begin{abstract} 
Federated learning provides a communication-efficient and privacy-preserving training process by enabling learning statistical models with massive participants without accessing their local data. Standard federated learning techniques that naively minimize an average loss function are vulnerable to data corruptions from outliers, systematic mislabeling, or even adversaries. 
In this paper, we address this challenge by proposing Auto-weighted Robust Federated Learning (\arfl), a novel approach that jointly learns the global model and the weights of local updates to provide robustness against corrupted data sources. We prove a learning bound on the expected loss with respect to the predictor and the weights of clients, which guides the definition of the objective for robust federated learning. We present an objective that minimizes the weighted sum of empirical risk of clients with a regularization term, where
the weights can be allocated by comparing the empirical risk of each client with the average empirical risk of the best $p$ clients. This method can downweight the clients with significantly higher losses, thereby lowering their contributions to the global model. We show that this approach achieves robustness when the data of corrupted clients is distributed differently from the benign ones. To optimize the objective function, we propose a communication-efficient algorithm based on the blockwise minimization paradigm. 
We conduct extensive experiments on multiple benchmark datasets, including CIFAR-10, FEMNIST, and Shakespeare, considering different neural network models. The results show that our solution is robust against different scenarios including label shuffling, label flipping, and noisy features, and outperforms the state-of-the-art methods in most scenarios.


\end{abstract}

%% file: body.tex
\section{Introduction}
\label{intro}

Federated learning \citep{mcmahan2017communication,li2020ditto,konevcny2016federateda}~has recently attracted more and more attention due to the increasing concern of user data privacy. In federated learning, the server trains a shared model based on data originating from remote clients such as smartphones and IoT devices, without the need to store and process the data in a centralized server. In this way, federated learning enables joint model training over privacy-sensitive data in a wide range of applications~\citep{yang2018applied, qiangyang2019federated}, including natural language processing~\citep{chen2019federated}, computer vision~\citep{luo2019real}, and speech recognition~\citep{guliani2020training}.

However, standard federated learning strategies fail in the presence of data corruption~\citep{rodriguez2020dynamic,mahloujifar2019data}. Data collected from different clients, or data sources (in this article we use the two terms interchangeably), may vary greatly in data quality and thus reduce the reliability of the learning task. Some of the clients may be unreliable or even malicious. For instance, distributed sensor networks are vulnerable to cybersecurity attacks, such as false data injection attacks~\citep{liu2019gaussian}. In crowdsourcing scenarios, the problem of noisy data is unignorable due to biased or erroneous workers~\citep{wais2010towards}. 
As a consequence, the clients can update their local models using corrupted data and send the parameters to the server. After averaging these harmful updates, the accuracy and the convergence of the shared global model can be compromised.
To mitigate this limitation, different robust learning approaches have been proposed in the literature. Some of these techniques rely on robust statistical estimations (e.g. geometric median estimators) to update the aggregated model~\citep{blanchard2017machine,pillutla2019robust,yin2018byzantine}. Although these techniques are widely used in traditional distributed learning scenarios with i.i.d. datasets, it is not straightforward to generalize them for non-i.i.d. data settings, i.e., some of the clients have significantly different data distribution than the others.
Other approaches require some trusted clients or samples to guide the learning~\citep{li2020communication,peng2019federated,sui2020feded} or detect the updates from corrupted clients~\citep{han2020robust,konstantinov2019robust}. Unfortunately, the credibility of these trusted clients and samples are usually not guaranteed. Since the data is stored locally in the clients, the server is insensible to the corruption behaviors and unable to measure the quality of data sources due to privacy constraints.

In this article, we aim to improve the robustness of federated learning when some of the clients are malicious, i.e., they actively corrupt the local training set by changing the features or the labels.
We propose a novel solution,  named Auto-weighted Robust Federated Learning (\arfl), to jointly learn the global model and the weights of local updates. More specifically, we first prove a learning bound on the expected risk with respect to the predictor and the weights of clients. Based on this theoretical insight, we present our objective that minimizes a weighted sum of the empirical risk of clients with a regularization term. We then theoretically show that the weights in the objective can be allocated by comparing the empirical risk of each client with the best $p$ clients. When the corrupted clients have significantly higher losses comparing with the benign ones, their contributions to the global model will be downweighted or even zero-weighted, so as to play less important roles in the global model. 
Therefore, by using \arfl, we can exclude potentially corrupted clients and keep optimizing the global model with the benign clients.
To solve the problem in federated learning settings, we further propose a communication-efficient optimization method based on the blockwise updating paradigm~\citep{zheng2019communication}. Through extensive experiments on multiple benchmark datasets (i.e., CIFAR-10, FEMNIST and Shakespeare) with different neural network models, we demonstrate the robustness of our approach compared with the state-of-the-art approaches~\citep{mcmahan2017communication, pillutla2019robust, blanchard2017machine, sattler2020byzantine}, showing up to $30\%$ improvement in model accuracy.

\input{related_work}

\input{preliminaries}

\section{Robust Federated Learning}
In this section, we describe our Auto-weighted Robust Federated Learning (\arfl) in detail. Specifically, we present a novel objective according to a learning bound with respect to both the predictor $h$ and the weights $\bm{\alpha}$. We analyze the robustness of the proposed objective against corruptions and optimize it with a federated learning based algorithm.

\subsection{Learning Bound}

In the following theorem, we present a learning bound on the risk on the weighted mixture distribution with respect to the predictor $h$ and the weights $\bm{\alpha}$. A proof is provided in Appendix \ref{proof1}.

\begin{restatable}{theorem}{mainthm}
\label{thm:main_bound}
We denote $\hat{\risk}_i(h)$ as the corresponding empirical counterparts of ${\risk}_{\mathcal{D}_i}(h)$. Assume that the loss function $\loss_{h}(\bm{z})$ is bounded by a constant $\mathcal{M} > 0$. Then, for any $\delta \in (0, 1]$, with probability at least $1 - \delta$ over the data, the following inequality holds:
\begin{align}
\notag \mathcal{L}_{\mathcal{D}_{\bm{\alpha}}}(h) \leq \sum_{i=1}^N \alpha_i{\hat{\risk}}_{i}(h) + 2\sum_{i=1}^N \alpha_i \mathcal{R}_i \left(\mathcal{H}\right)  + 3 \sqrt{\frac{\log\left(\frac{4}{\delta}\right)\mathcal{M}^2}{2}}
\sqrt{\sum_{i=1}^n\frac{\alpha_i^2}{m_i}},
\label{eqn:final_bound}
\end{align}
where, for each client $i = 1, \ldots, N$,
\begin{equation*}
\mathcal{R}_i \left(\mathcal{H}\right) = \mathbb{E}_{\sigma}\left(\sup_{f\in\mathcal{H}}\left(\frac{1}{m_i}\sum_{j=1}^{m_i}\sigma_{i,j}\loss_f(z_{i,j})\right)\right)
\end{equation*}
\noindent and $\sigma_{i,j}$s are independent uniform random variables taking values in $\{-1, +1\}$. These variables are called Rademacher random variables~\citep{yin2019rademacher}.
\end{restatable}

Although the Rademacher complexities in the bound are functions of both the underlying distribution and the hypothesis class~\citep{yin2019rademacher}, in practice one usually works with a computable upper bound of $\mathcal{R}_i\left(\mathcal{H}\right)$  that is distribution-independent (e.g. using VC dimension)~\citep{shalev2014understanding, konstantinov2019robust, bousquet2003introduction}. In our setting the hypothesis space $\mathcal{H}$ is fixed and hence these bounds would be identical for all $i$. Therefore, we expect the $\mathcal{R}_i\left(\mathcal{H}\right)$ to be of similar order for all clients and the impact of $\bm{\alpha}$ in the second term to be negligible. We refer to Konstantinov et al.~\cite[Appendix~B.1]{konstantinov2019robust} for detailed explanations and some examples on this point. In general, we can ignore this term and concentrate on optimizing the remaining terms.

The rest of the terms for the learning bound in Theorem~\ref{thm:main_bound} thus suggest a trade-off between reducing the weighted sum of the empirical losses and minimizing a weighted norm of the weights $\sqrt{\sum_{i=1}^n\frac{\alpha_i^2}{m_i}}$. Note that reducing the weighted sum of the empirical losses will encourage trusting the clients who provide data with the smallest loss but it may lead to large and sparse weights, which will make the model fit only very few local datasets. On the contrary, minimizing the norm will increase the smoothness of the weight distribution, but it may also increase the weighted sum of empirical losses. 
To control this trade-off, we introduce a tuning hyperparameter $\lambda$ to the weighted norm. For the convenience of derivation, the square root on the sum of weights can be removed by tuning $\lambda$.

\subsection{Problem Formulation}
We present our problem formulation in the following. The learning bound derived above suggests minimizing the following objective with respect to the model parameters $\bm{w}$ together with the weights $\bm{\alpha}$:
\begin{align}
\label{eq:fl:main}
 \min_{\bm{w}, \bm{\alpha} } \quad & \sum_{i=1}^N \alpha_i{\hat{\risk}}_{i}(\bm{w})  +  \frac{\lambda}{2}{\sum_{i=1}^N \frac{\alpha_i^2}{m_i}}, \\
\notag s.t. \quad & \bm{\alpha} \in \mathbb{R}_+^n, \bm{1}^{\top}\bm{\alpha} = 1,
\end{align} where $\bm{w}$ is a vector of parameters defining a predictor $h$. Here and after we use notation ${\hat{\risk}}_{i}(\bm{w})$ to replace ${\hat{\risk}}_{i}(h)$, representing the empirical risk of hypothesis $h$ (corresponding to $\mathbf{w}$) on client~$i$. Note that removing the square root of the last term does not affect the optimal solution, since the total sum ${\sum_{i=1}^N \frac{\alpha_i^2}{m_i}}$ is the same in both forms, and the square root is simply a scaling for importance of the total sum. In practice, we can achieve the equivalent purpose by adjusting $\lambda$.

The second term of the objective is small whenever the weights are distributed proportionally to the number of samples of the client. As $\lambda \to \infty$, we have $\alpha_i(\bm{w}) = \frac{m_i}{M}$, which means that all clients are assigned with weights proportional to their number of training samples and the model minimizes the empirical risk over all the data, regardless of the losses of the clients. Thus, the objective becomes the same as the standard \FedAvg in Eq.~(\ref{emp_risk_m}). In contrast, as $\lambda \to 0$, the regularization term in~Eq.~(\ref{eq:fl:main}) vanishes, so that the client with the lowest empirical risk will dominate the objective by setting its weight to 1. $\lambda$ thus acts as a form of regularization by encouraging the usage of information from more clients.
\subsection{Robustness of the Objective}

We now study the optimal weights $\bm{\alpha}$ of the problem in \
Eq.~(\ref{eq:fl:main}) to understand how the objective yields robustness against corrupted data sources. We notice that the objective is a convex quadratic program problem over $\bm{\alpha}$. Given that $\bm{\alpha} = N^{-1}\bm{1}$ is a strictly feasible point, the problem satisfies Slater's condition, which indicates the strong duality of the problem. Thus, the optimal weights $\bm{\alpha}$ can be obtained using the Karush-Kuhn-Tucker (KKT) conditions~\citep{boyd2004convex}. Here we give the closed-form solution in~Theorem~\ref{thm:alpha}. The detailed proof is provided in Appendix~\ref{proof2}.

\begin{restatable}{theorem}{alphathm}
\label{thm:alpha}
		 For any $\bm{w}$, when $\lambda > 0$ and $\{\hat{\risk}_i(\bm{w})\}_{i=1}^N$ are sorted in increasing order: $\hat{\risk}_1(\bm{w}) \le \hat{\risk}_2(\bm{w}) \le ... \le \hat{\risk}_N(\bm{w})$, by setting:
		\begin{equation}
		    \label{nonzero}
    		    p = \argmax_k\{1  + 
          \frac{M_k( \overline{\risk}_k(\bm{w}) -\hat{\risk}_k(\bm{w})) }{\lambda} > 0\},
		\end{equation}
		where $M_k = \sum_{i=1}^k m_i$, 
		\begin{equation}
		    \label{avg}
		    \overline{\risk}_k(h) = \frac{ \sum_{i=1}^{k}m_i \hat{\risk}_i(\bm{w})} {M_k}
		\end{equation}
		is the average loss over the first $k$ clients that have the smallest empirical risks. Then the optimal $\bm{\alpha}$ to the problem (\ref{eq:fl:main}) is given by:
		\begin{equation}
          \label{inner_solution}
          \alpha_i(\bm{w}) = \frac{m_i}{M_p}[ 1  + 
          \frac{M_p( \overline{\risk}_p(\bm{w}) -\hat{\risk}_i(\bm{w})  )}{\lambda}]_{+},
		\end{equation}
		where 	$[\cdot]_+ = max(0, \cdot)$.
\end{restatable}

When $\lambda \in (0, \infty)$, plugging $\alpha_i(\bm{w})$ back into Eq.~(\ref{eq:fl:main}) yields the equivalent concentrated objective
{
\footnotesize
\begin{equation}
     \sum_{i=1}^N  [\frac{m_i}{M_p} + \frac{m_i (\overline{\risk}_p(\bm{w}) - \hat{\risk}_i(\bm{w})}{\lambda}]_{\tiny +}(\hat{\risk}_i(\bm{w}) + \frac{1}{2}[\frac{\lambda}{M_p} + \overline{\risk}_p(\bm{w}) - \hat{\risk}_i(\bm{w})]_{_{_+}}),
\end{equation}
} which consists of $N$ components and each is related to the empirical risk of the corresponding client and the average loss over the first $p$ clients with the smallest losses, which is called the $p$-average loss. An intuitive interpretation is that the $p$ clients act as a consensus group to reallocate the weights and encourage trusting clients that provide empirical losses that are smaller than the average while downweighting the clients with higher losses. Given a suitable $\lambda$, more benign clients will be in the consensus group to dominate the model and exclude the outliers.

In a situation where the majority of clients are benign, $\overline{\risk}_p(\bm{w})$ becomes relatively low as the benign clients achieve the minimum empirical risk. The components with corrupted datasets will be downweighted as they have higher losses. Especially, the $i$-th component becomes zero  when $\hat{\risk}_i(\bm{w}) \ge \frac{\lambda}{M_p} + \overline{\risk}_p(\bm{w})$, which means that a client is considered to be corrupted and does not contribute to the objective if its empirical risk is significantly larger than the $p$-average loss, where the threshold $\frac{\lambda}{M_p} + \overline{\risk}_p(\bm{w})$ is controlled by $\lambda$. From Eq.~(\ref{nonzero}) we can also conclude that the optimal solution has only $p$ non-zero components and the remaining components will be exactly zero.

On the other hand, if the corrupted clients try to bias the model to fit their corrupted datasets, the $p$-average loss $\overline{\risk}_p(\bm{w})$ becomes higher because the model does not fit the samples from the majority. The threshold $\frac{\lambda}{M_p} + \overline{\risk}_p(\bm{w})$ will also be enlarged, which makes $\alpha_i(\bm{w})$ fail to downweight the component with high losses. Thus, the optimization problem in Eq.~(\ref{eq:fl:main}) will exceed the minimum.

\subsection{Blockwise Minimization Algorithm}
To solve the robust learning problem in federated learning settings, we propose an optimization method based on the blockwise updating paradigm, which is guaranteed to converge to a critical point when the parameter set is closed and convex~\citep{grippo2000convergence}. The key idea is to divide the problem into two parts. One sub-problem for estimating the model parameters and the other sub-problem for automatically weighting the importance of client updates. Then we minimize the objective iteratively w.r.t. one variable each time while fixing the other one. 

The pseudocode of the optimization procedure is given in Algorithm~\ref{alg:ARFL}. At the beginning of the algorithm, we initialize $\hat{\risk} = [\hat{\risk}_1,\hat{\risk}_2, ..., \hat{\risk}_N]^\top$ by broadcasting the initial global model $\bm{w_0}$ to each client to measure its training loss and return it to the server.

\begin{algorithm}
\setlength{\abovedisplayskip}{0pt}
\setlength{\belowdisplayskip}{0pt}
\setlength{\abovedisplayshortskip}{0pt}
\setlength{\belowdisplayshortskip}{0pt}
\caption{Optimization of \arfl}
\label{alg:ARFL}
\begin{algorithmic}[1]
    \SUB{Server executes:}
		\STATE Initialize $\bm{w}_0, \hat{\risk}$, $\bm{\alpha}$
		\FOR{each round $t = 1, 2, \dots$} 
		\STATE Select a subset $S_t$ from $N$ clients at random \label{ag:select}
		\STATE Broadcast the global model $\bm{w}_t$ to selected clients $S_t$ \label{ag:broad}
		\FOR{each client $i \in S_t$ \textbf{in parallel}}
		\STATE $\bm{w}_{t+1}^i,\hat{\risk}_i \leftarrow \text{ClientUpdate}(i, \bm{w}_t)$  \label{ag:back}

		\ENDFOR
		\STATE Update $\bm{w}_{t+1}$ according to Eq.~(\ref{eq:fl:aggregation}) \label{ag:agg}
		\STATE Update $\bm{\alpha}$ according to Theorem~\ref{thm:alpha}
		\ENDFOR
		
    \STATE
		
	\SUB{ClientUpdate($i, \bm{w}$):}\ \ \  // \emph{Run on client $i$}
	    \STATE $\risk_i \leftarrow$ (evaluate training loss using training set)
    	\STATE $\mathcal{B} \leftarrow$ (split local training set into batches of size $\lbs$)
    	\FOR{each local epoch $i$ from $1$ to $\lepochs$} \label{local_begin}
    	\FOR{batch $b \in \mathcal{B}$}
    	\STATE $\bm{w} \leftarrow \bm{w} - \eta \triangledown \loss(\bm{w}; b)$
    	\ENDFOR
    	\ENDFOR \label{local_end}
    	\STATE return $\bm{w}$ and $\hat{\risk}_i$ 
\end{algorithmic}
\end{algorithm}

\textbf{Updating $\bm{w}$.} When $\bm{\alpha}$ is fixed, similar to the standard \fedavg approach, at round $t$, the server selects a subset $S_t$ of clients at random (Line~\ref{ag:select}) and broadcasts the global model $\bm{w}_t$ to the selected clients (Line~\ref{ag:broad}). For each client $i$, it firstly evaluates its training loss $\risk_i$ using its local dataset. Then, the model parameters can be updated by local computation with a few steps of SGD (Line~\ref{local_begin}-\ref{local_end}), after which the client uploads the new model parameters $\bm{w}$ along with $\risk_i$ to the server (Line~\ref{ag:back}). While at the aggregation step, the server assembles the global model as:
\begin{equation}\label{eq:fl:aggregation}
    \bm{w}_{t+1} \leftarrow \sum_{i \in S_t} \frac{\alpha_i}{\sum_{i \in S_t}\alpha_i} \bm{w}_{t+1} ^ i.
\end{equation}

\textbf{Updating $\bm{\alpha}$.} When $\bm{w}$ is fixed, we update $\bm{\alpha}$ using Eq.~(\ref{inner_solution}) in Theorem~\ref{thm:alpha}. Intuitively, in order to update $\bm{\alpha}$, the server should broadcast the updated model parameters to all clients to obtain their training loss before updating $\bm{\alpha}$. Unfortunately, such behavior might significantly increase the burden of the communication network. To improve communication efficiency, we only update the losses from those selected clients while keeping the others unchanged.
In other words, the weights of clients are reallocated according to their latest empirical losses. If a client is not selected in the current round, the last updated loss is used instead.

\section{Experimental Results}
\label{experiments}
\label{exper}

\subsection{Experimental Setup}
\label{sess_setup}
\begin{table*}[t]
	\centering
	\caption{Dataset description and parameters}
	{\fontsize{8.0pt}{8.0pt}\selectfont   
	\begin{tabular}{@{}lcccccccccc@{}}
		\toprule
		Dataset \T & \#Classes & \#Clients& \#Samples & i.i.d. & Model used & $l_r$ & $E$ & Batch size & $|S_t|$ & \#Rounds\\ 
		\midrule
		CIFAR-10  \T& 10 & 100 & 60,000& Yes & CNN & 0.01 & 5 & 64 & 20 & 2000\\
		FEMNIST  \T& 62 & 1039 & 236,500& No & CNN & 0.01 & 20 & 64 & 32& 2000\\
		Shakespeare  \T& 80 & 71 & 417,469 &  No & LSTM & 0.6 & 1 & 10 & 16& 100\\ 
		\bottomrule
	\end{tabular}
	\vspace*{2.5mm}
	\label{table:dataset}
	}
\end{table*}

We implement \arfl and the considered baseline methods in TensorFlow~\citep{abadi2016tensorflow}~Version 2.3\footnote{Code is available at \url{https://github.com/lishenghui/arfl}}, simulating a federated learning system with one server and $N$ clients. We perform our experimental evaluation on three datasets that are commonly used in previous work~\citep{li2019fair,wang2020federated,mcmahan2017communication}, 
namely CIFAR-10~\citep{krizhevsky2009learning}, FEMNIST~\citep{cohen2017emnist,caldas2018leaf}, and Shakespeare~\citep{caldas2018leaf,mcmahan2017communication}. Their basic information is listed in Table~\ref{table:dataset}. For the CIFAR-10 dataset, we consider an i.i.d. partition where each local client has approximately the same amount of samples and in proportion to each of the classes. We use the original test set in CIFAR-10 as our global test set for comparing the performance of all methods. For the Shakespeare and FEMNIST datasets, we treat each speaking role or writer as a client and randomly sample subsets of all clients. We assume that data distributions vary among clients in the raw data, and hence we regard this sampling process as non-i.i.d.. The raw data in FEMNIST and Shakespeare is preprocessed using the popular benchmark LEAF~\citep{caldas2018leaf}, where the data on each local client is partitioned into an 80\% training set and a 20\% testing set.
The details of the network models we use in the experiments are as follows:
\begin{itemize}
  \item The model for CIFAR-10 is a Convolutional Neural Network (CNN) chosen from Tensorflow's website\footnote{\url{https://www.tensorflow.org/tutorials/images/cnn}}, which consists of three 3x3 convolution layers  (the first with 32 channels, the second and third with 64, the first two followed by 2x2 max pooling), a fully connected layer with 64 units and ReLu activation, and a final softmax output layer. To improve the performance, data augmentation  (random shift and flips) is used in this dataset~\citep{wang2020federated}.
  \item For the FEMNIST dataset, we train a CNN with two 5x5 convolution layers (the first with 32 channels, the second with 64, each followed by 2x2 max pooling), a fully connected layer with 126 units and ReLu activation, and a final softmax output layer.
  \item For the Shakespeare dataset, we learn a character-level language model to predict the next character over \textsl{the Complete Works of Shakespeare}~\citep{shakespeare2007complete}. The model takes a series of characters as input and embeds each of these into an 8-dimensional space. The embedded features are then processed through two stacked 
Long Short-Term Memory (LSTM) layers, each with 256 nodes and a dropout value of 0.2. Finally, the output of the second LSTM layer is sent to a softmax output layer with one node per character. 
\end{itemize}

For each dataset we consider four different scenarios: 1) \textbf{Normal operation} (\textit{clean}): we use the original datasets without any corruption. 2) \textbf{Label shuffling} (\textit{shuffling}): the labels of all samples are shuffled randomly  in each corrupted client. 3) \textbf{Label flipping} (\textit{flipping}): the labels of all samples are switched to a random one in each corrupted client, which means that all labels of training samples are flipped as the same one for each corrupted client. 4) \textbf{Noisy clients} (\textit{noisy}): for CIFAR-10 and FEMNIST datasets, we normalize the inputs to the interval [0, 1]. In this scenario, for the selected noisy clients we add Gaussian noise to all the pixels, so that $x \leftarrow x + \epsilon $, with $ \epsilon \sim N (0, 0.7)$. Then we normalize the resulting values again to the interval $[0, 1]$. For the Shakespeare dataset, we randomly select half of the characters and shuffle them so that the input sentence might be disordered. For each corruption scenario, we set 30\% and 50\% of the clients to be corrupted clients (i.e. providing corrupted data).

\begin{table*}[t!]
	\centering
	{
		\fontsize{6.0pt}{6.pt}\selectfont   
		\setlength{\tabcolsep}{6pt} 
		\renewcommand{\arraystretch}{1.2} 
				\caption{Averaged test accuracy over five random seeds for \fedavg, \rfa, \MKrum, \CFL and \arfl in four different scenarios. In the \textit{shuffling} and \textit{flipping} scenarios, \arfl significantly outperforms the others. In the \textit{clean} and \textit{noisy} scenario, \fedavg, \rfa, \MKrum and \arfl achieve similar accuracy.}
		\begin{tabular}{|l|l|ll|ll|ll|}
			\hline
			
			\multicolumn{1}{|c|}{\textbf{CIFAR-10}} \T&
			\multicolumn{1}{c|}{\textbf{Clean}} &
			\multicolumn{2}{c|}{\textbf{Shuffling}} &
			\multicolumn{2}{c|}{\textbf{Flipping}} &
			\multicolumn{2}{c|}{\textbf{Noisy}}  \\ \hline
			
			\multicolumn{1}{|c|}{Corr. Per.} \T &
			\multicolumn{1}{c|}{-} &
			\multicolumn{1}{c}{30\%} &
			\multicolumn{1}{c|}{50\%} &
			\multicolumn{1}{c}{30\%} &
			\multicolumn{1}{c|}{50\%} &
			\multicolumn{1}{c}{30\%} &
			\multicolumn{1}{c|}{50\%} \\ \hline
			\FedAvg~\citep{mcmahan2017communication} & $73.59 \pm 0.44$ & $61.17 \pm 1.81$ & $47.00 \pm 7.51$ & $65.01 \pm 2.38$ & $51.75 \pm 7.75$ & $\mathbf{73.75 \pm 0.49}$ & $73.61 \pm 0.53$ \\
			\RFA~\citep{pillutla2019robust}  & $71.36 \pm 0.47$ & $57.86 \pm 3.22$ & $40.26 \pm 9.14$ & $55.47 \pm 5.17$ & $40.91 \pm 11.06$ & $73.74 \pm 0.52$ & $\mathbf{73.69 \pm 0.63}$ \\
			\MKrum~\citep{blanchard2017machine}  & $67.03 \pm 0.93$ & $59.27 \pm 9.34$ & $52.32 \pm 14.90$ & $60.21 \pm 5.73$ & $47.96 \pm 10.25$ & $73.41 \pm 0.69$ & $73.49 \pm 0.49$ \\
			\CFL~\citep{sattler2020byzantine}  & $71.68 \pm 0.36$ & $52.54 \pm 1.71$ & $50.29 \pm 1.95$ & $52.87 \pm 1.07$ & $51.67 \pm 0.92$  & $54.97 \pm 1.14$ & $55.26 \pm 1.96$ \\
			\hline
			\ARFL(ours) & $73.42 \pm 0.40$ & $\mathbf{71.68 \pm 1.01}$ & $\mathbf{69.66 \pm 0.73}$ & $\mathbf{71.78 \pm 0.53}$ & $\mathbf{70.25 \pm 0.56}$  & $73.48 \pm 0.56$ & $73.29 \pm 0.79$ \\
			\hline
		\end{tabular}
		\begin{tabular}{|l|l|ll|ll|ll|}
			\hline
			
			\multicolumn{1}{|c|}{\textbf{FEMNIST}} \T&
			\multicolumn{1}{c|}{\textbf{Clean}} &
			\multicolumn{2}{c|}{\textbf{Shuffling}} &
			\multicolumn{2}{c|}{\textbf{Flipping}} &
			\multicolumn{2}{c|}{\textbf{Noisy}}  \\ \hline
			
			\multicolumn{1}{|c|}{Corr. Per.} \T &
			\multicolumn{1}{c|}{-} &
			\multicolumn{1}{c}{30\%} &
			\multicolumn{1}{c|}{50\%} &
			\multicolumn{1}{c}{30\%} &
			\multicolumn{1}{c|}{50\%} &
			\multicolumn{1}{c}{30\%} &
			\multicolumn{1}{c|}{50\%} \\ \hline
			\FedAvg~\citep{mcmahan2017communication} & $82.12 \pm 0.20$ & $61.91 \pm 21.33$ & $39.69 \pm 20.80$ & $70.19 \pm 10.17$ & $48.53 \pm 23.49$ & $79.94 \pm 0.36$ & $78.27 \pm 0.47$ \\
			\RFA~\citep{pillutla2019robust} & $82.11 \pm 0.32$ & $74.36 \pm 7.52$ & $52.02 \pm 22.51$ & $73.80 \pm 7.49$ & $50.75 \pm 19.91$ & $80.45 \pm 0.30$ & $79.21 \pm 0.41$ \\
			\MKrum~\citep{blanchard2017machine}  & $79.38 \pm 0.41$ & $57.51 \pm 21.17$ & $42.40 \pm 24.84$ & $78.57 \pm 4.83$  & $67.10 \pm 7.35$  & $\mathbf{81.52 \pm 0.53}$ & $\mathbf{79.80 \pm 0.22}$ \\
			\CFL ~\citep{sattler2020byzantine}  & $82.18 \pm 0.30$ & $81.24 \pm 0.47$ & $36.03 \pm 36.38$ & $81.22 \pm 0.36$ & $65.54 \pm 26.94$ & $80.13 \pm 0.70$ & $79.21 \pm 0.64$ \\
			\hline
			\ARFL(ours)   & $\mathbf{82.32 \pm 0.19}$ & $\mathbf{81.60 \pm 0.31}$ & $\mathbf{81.35 \pm 0.43}$  & $\mathbf{81.87 \pm 0.22}$ & $\mathbf{81.30 \pm 0.24}$  & $80.71 \pm 0.28$ & $79.40 \pm 0.45$ \\
			\hline
		\end{tabular}
		\begin{tabular}{|l|l|ll|ll|ll|}
			\hline
			
			\multicolumn{1}{|c|}{\textbf{Shakespeare}} \T&
			\multicolumn{1}{c|}{\textbf{Clean}} &
			\multicolumn{2}{c|}{\textbf{Shuffling}} &
			\multicolumn{2}{c|}{\textbf{Flipping}} &
			\multicolumn{2}{c|}{\textbf{Noisy}}  \\ \hline
			
			\multicolumn{1}{|c|}{Corr. Per.} \T &
			\multicolumn{1}{c|}{-} &
			\multicolumn{1}{c}{30\%} &
			\multicolumn{1}{c|}{50\%} &
			\multicolumn{1}{c}{30\%} &
			\multicolumn{1}{c|}{50\%} &
			\multicolumn{1}{c}{30\%} &
			\multicolumn{1}{c|}{50\%} \\ \hline
			\FedAvg~\citep{mcmahan2017communication} & $53.80 \pm 0.33$ & $51.98 \pm 0.48$ & $47.70 \pm 4.96$ & $52.08 \pm 0.39$ & $41.85 \pm 16.18$ & $51.85 \pm 0.56$ & $50.43 \pm 1.19$ \\
			\RFA~\citep{pillutla2019robust}  & $\mathbf{54.27 \pm 0.41}$ & $50.16 \pm 1.28$ & $32.49 \pm 13.81$ & $50.50 \pm 1.02$ & $23.84 \pm 21.78$ & $\mathbf{52.17 \pm 0.50}$ & $50.69 \pm 1.04$ \\
			\MKrum~\citep{blanchard2017machine}  & $50.81 \pm 0.85$ & $40.38 \pm 7.44$ & $24.46 \pm 6.88$ & $44.95 \pm 2.43$ & $16.11 \pm 15.46$ & $48.19 \pm 0.40$ & $45.67 \pm 0.46$ \\
			\CFL ~\citep{sattler2020byzantine} & $54.01 \pm 0.34$ & $49.76 \pm 4.47$ & $43.68 \pm 12.68$ & $51.09 \pm 1.36$ & $37.30 \pm 19.76$ & $51.98 \pm 1.03$ & $50.38 \pm 1.39$ \\
			\hline
			\ARFL(ours)  & $53.52 \pm 0.32$ & $\mathbf{52.85 \pm 0.49}$ & $\mathbf{51.61 \pm 0.68}$  & $\mathbf{52.82 \pm 0.48}$ & $\mathbf{51.74 \pm 0.69}$  & $52.09 \pm 1.27$ & $\mathbf{50.98 \pm 0.75}$ \\
			\hline
		\end{tabular}
		\label{table:summary}
	}
\end{table*}

We empirically tune the hyper-parameters on \arfl and use the same values in all experiments of each dataset. We use the parameter setups in Table~\ref{table:dataset}, unless specified otherwise. Following the standard setup, we use SGD and train for $E$ local epochs with local learning rate $l_r$. A shared global model is trained by all participants, a subset $S_t$ is randomly selected in each round of local training, and $|S_t|$ is the size of $S_t$. By default, we use a large $\lambda$ $(\lambda = 10000 \times M)$ for clean data, and use a relatively small $\lambda$ $(\lambda = M)$ for all corruptions, where $M$ is the total amount of training samples. We repeat every experiment five times with different random seeds for data corruption and client selection, and evaluate the accuracy of the learned model with the clean test set.

We compare the performance of \arfl with the following state-of-the-art solutions: 
\begin{itemize}
    \item \textbf{FedAvg}~\citep{mcmahan2017communication}. The standard Federated Averaging aggregation approach that just calculates the weighted average of the parameters from local clients.
    \item \textbf{RFA}~\citep{pillutla2019robust}. A robust aggregation approach that minimizes the weighted Geometric Median (GM) of the parameters from local clients. A smoothed Weiszfeld’s algorithm is used to compute the approximate GM.
    

    \item \textbf{MKrum (Multi-Krum)}~\citep{blanchard2017machine}. A Byzantine tolerant aggregation rule. Note that this approach tolerates some Byzantine failures such as completely arbitrary behaviors from local updates.
    \item \textbf{CFL}~\citep{sattler2020byzantine}. A Clustered Federated Learning (CFL) approach that separates the client population into different groups based on the pairwise cosine similarities between their parameter updates, where the clients are partitioned into two groups, i.e., benign clients and corrupted clients. 
\end{itemize}


\subsection{Robustness and Convergence}

\begin{figure*}
	\begin{center}
	\begin{subfigure}[b]{0.5\linewidth}
         \centering
         \includegraphics[width=\textwidth]{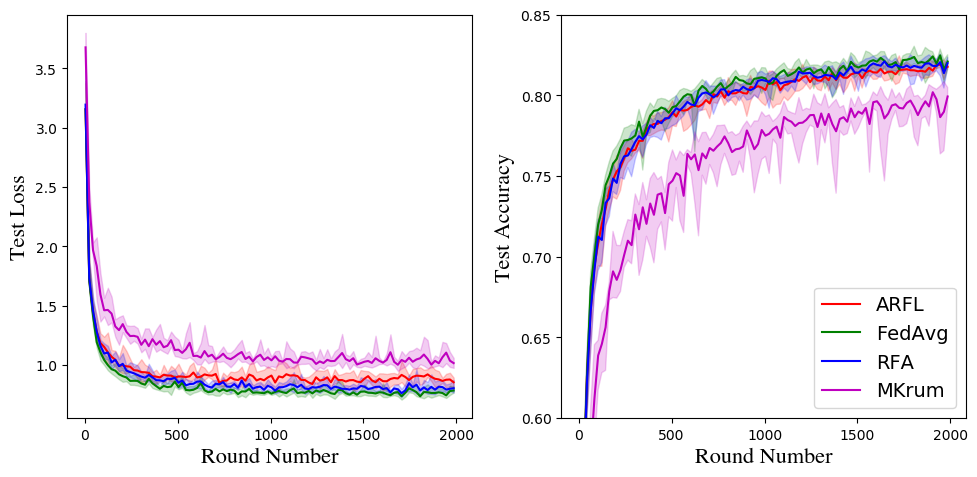}
         \caption{Clean}
     \end{subfigure}%
     \begin{subfigure}[b]{0.5\linewidth}
         \centering
         \includegraphics[width=\textwidth]{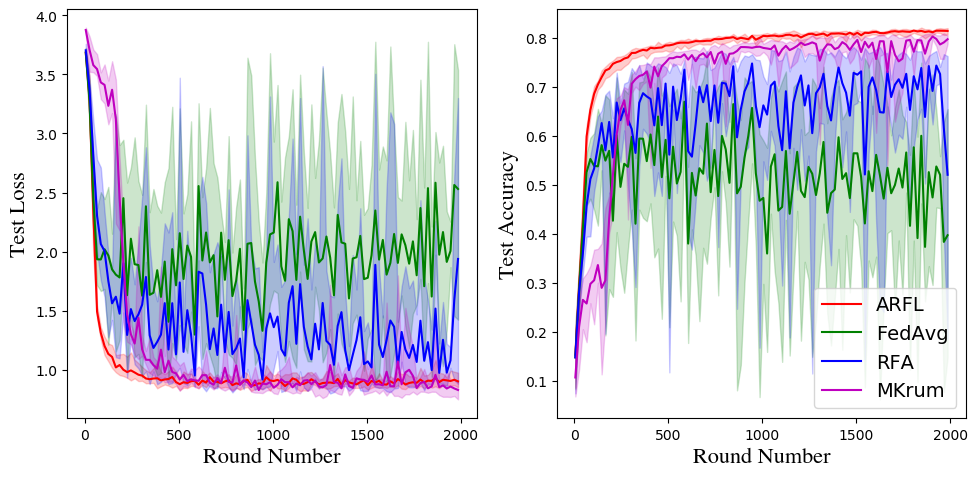}
         \caption{Shuffling}
     \end{subfigure}
     \begin{subfigure}[b]{0.5\linewidth}
         \centering
         \includegraphics[width=\textwidth]{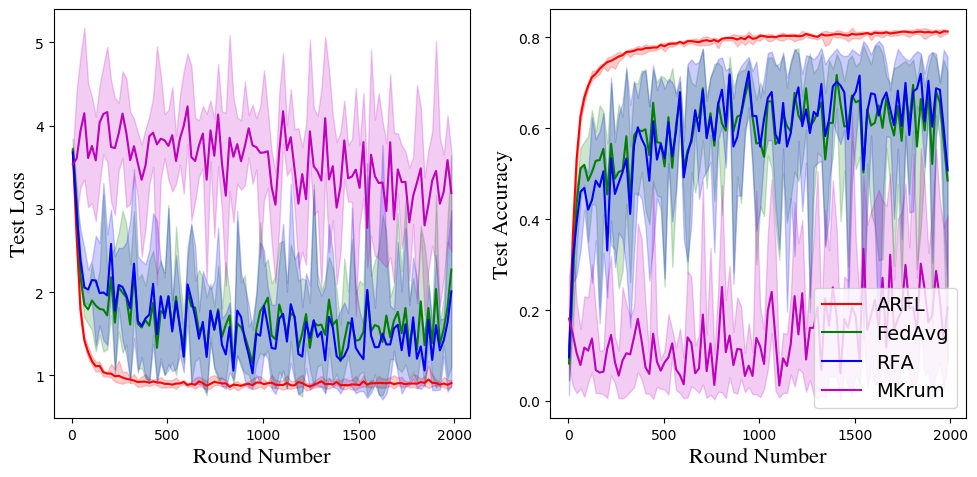}
         \caption{Flipping}
     \end{subfigure}%
     \begin{subfigure}[b]{0.5\linewidth}
         \centering
         \includegraphics[width=\textwidth]{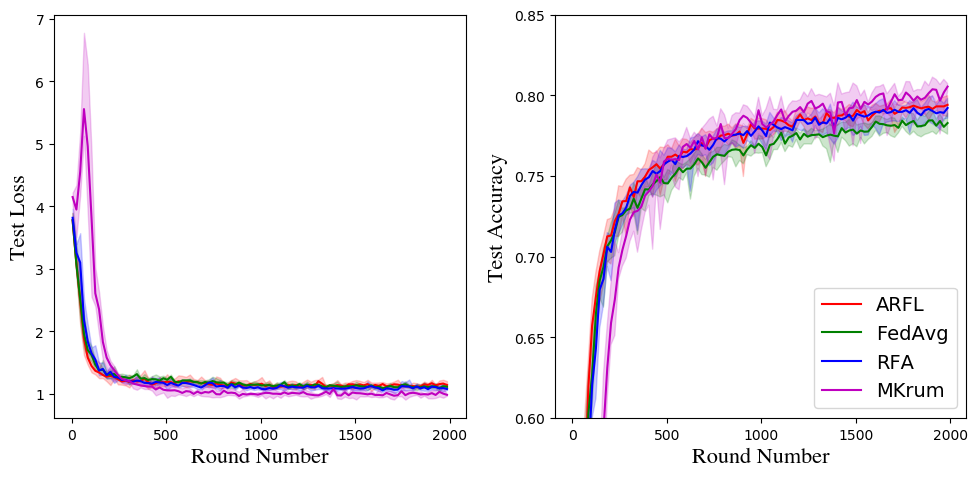}
         \caption{Noisy}
         \label{fig:noisy_con}
     \end{subfigure}
		\vskip -0.05in \caption{Test loss and accuracy vs.
			round number for \fedavg, \rfa, \MKrum and \arfl on the FEMNIST dataset with $50\%$ clients with different corruption scenarios. In the \textit{shuffling} and \textit{flipping} corruption scenarios, \arfl converges to the highest accuracy among the four approaches.}
		\label{fig:convergence}
	\end{center}
\end{figure*} 

Firstly, we show that \arfl is a more robust solution by comparing the average test accuracy in Table~\ref{table:summary}. The results show that \arfl is robust in all data corruption scenarios and corruption levels. It achieves the highest test accuracy in most scenarios compared with the other approaches. \arfl achieves significantly higher test accuracy in the \textit{shuffling} and \textit{flipping} corruption scenarios compared with all the existing methods, which is especially noticeable in the case of \textit{flipping} corruption with the level of $50\%$, where the test accuracy for the CIFAR-10, FEMNIST and Shakespeare datasets are $70.25\%, 81.30\%$ and $51.74\%$, which are $18.5\%, 14.2\%$, and $9.89\%$ higher than that of the best of existing methods, respectively. In the \textit{clean} and \textit{noisy} scenarios, \arfl's test accuracy is very close to the best method in the comparison. 

As expected, \fedavg's performance is significantly affected by the presence of corrupted clients, especially in \textit{shuffling} and \textit{flipping} scenarios. Furthermore, \MKrum also shows poor performance in \textit{shuffling} and \textit{flipping} scenarios of all datasets. \rfa works well for the FEMNIST dataset, but worse than \fedavg in the \textit{shuffling} and \textit{flipping} scenarios for the CIFAR-10 and Shakespere datasets. It is also interesting to observe that \CFL works well for the FEMNIST and Shakespeare datasets under $30\%$ corruption level, but the accuracy decreases significantly when the corruption level is $50\%$. The reason is that when half of the clients are corrupted, \CFL fails to identify which group of clients are corrupted. These results demonstrate that \arfl offers better performance than the existing approaches across the corruption scenarios we consider. Note that our approach can handle even higher corruption rates in those scenarios. For example, using the FEMNIST data set with 70\% corrupted clients we still achieve an accuracy above $79\%$. However, it is also noticeable that if multiple clients try to bias their data to the same distribution (i.e., colluding corruption), our approach is unable to handle such a high corruption rate.

Next, we study the convergence of the approaches by comparing the test loss and accuracy of the global model versus the number of training rounds in Figure~\ref{fig:convergence}, where 50\% of the clients are corrupted. As discussed before, \CFL is unable to handle such a high corruption level. Therefore we only compare the remaining four approaches. The shaded areas denote the minimum and maximum values over five repeated runs. The figure shows that all approaches converge to a good solution in the \textit{clean} and \textit{noisy} scenarios. However, in the \textit{clean} scenario, the test loss of \arfl is slightly higher than the others. The reason is that the global model could bias towards some of the local updates, which can be avoided by increasing $\lambda$. Furthermore, all the \rfa, \fedavg and \MKrum approaches diverge in the \textit{shuffling} and \textit{flipping} corruption scenarios, which indicates that the local data corruption harms the global model during their training process. On the contrary, our \arfl method is able to converge with high accuracy, since it is able to lower the contribution of the corrupted clients. In the \textit{clean} and \textit{noisy} scenarios, we observe that \MKrum converges slower, as it only uses a subset of selected updates to aggregate the global model.

\begin{figure}
	\begin{center}
		\begin{minipage}{.48\textwidth}
			\centering
			\includegraphics[width=\linewidth]{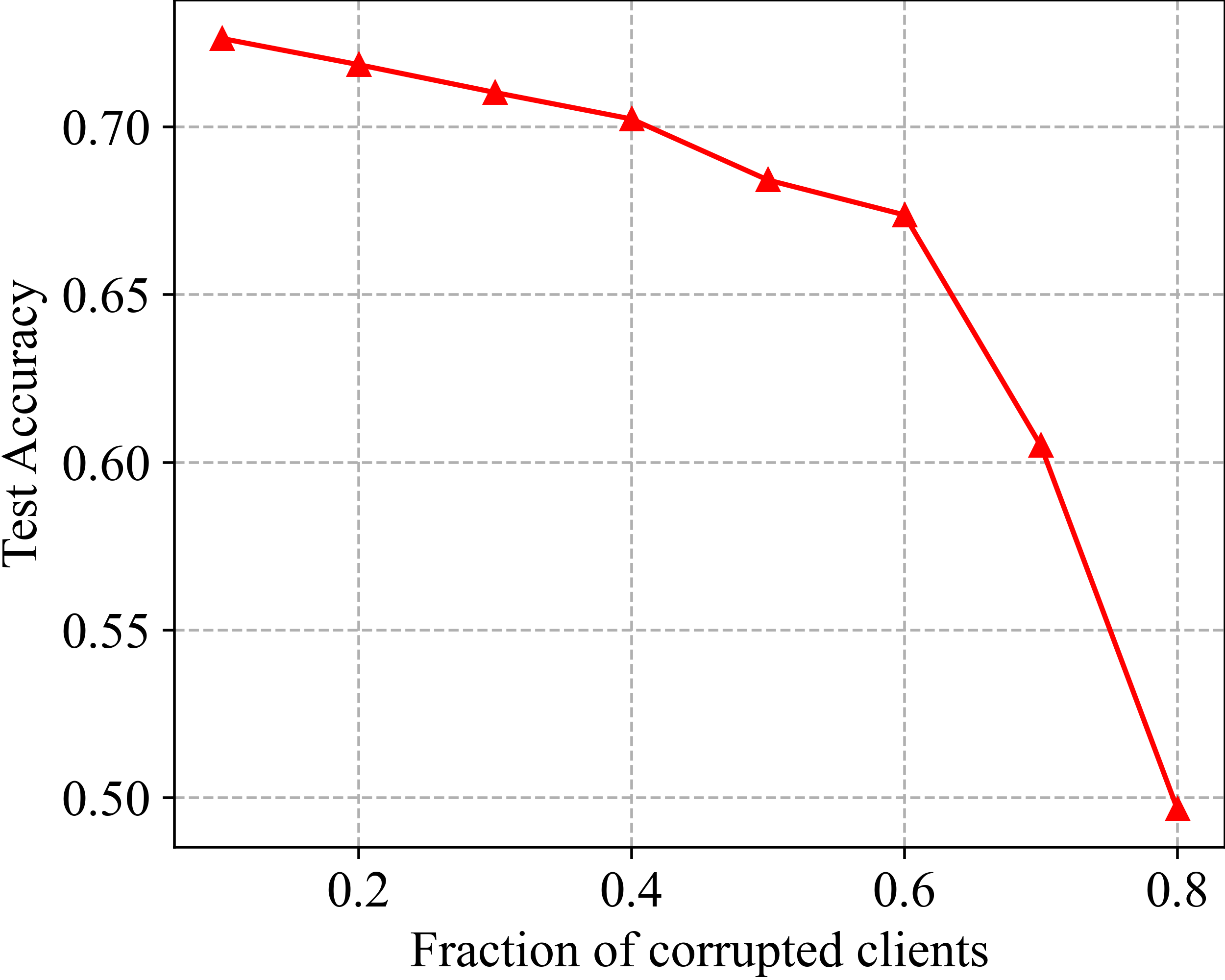}
			\caption{Impact of fraction of corrupted clients on test accuracy of \arfl using CIFAR-10. The accuracy decreases as the fraction of corrupted clients increases, especially when 60\% clients are corrupted.}
			\label{maximum}
		\end{minipage}%
		\hspace{0.3cm}
		\begin{minipage}{.48\textwidth}
			\centering
			\includegraphics[width=\linewidth]{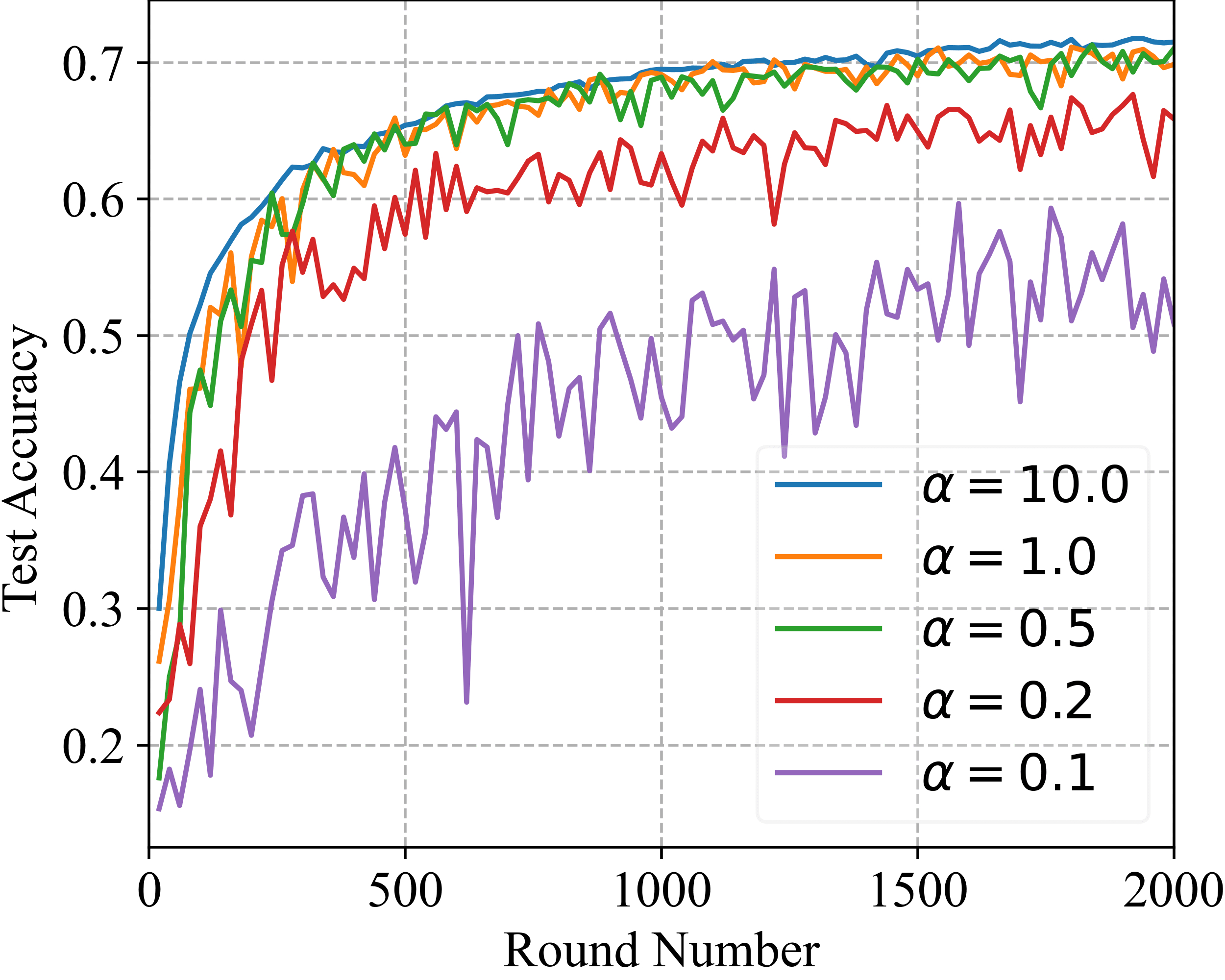}
			\caption{Learning curves of \ARFL on CIFAR-10 with different non-i.i.d settings, where smaller $\alpha$ indicates stronger non-i.i.d. data distribution. The curves fluctuates as we increase the degree of non-i.i.d.}
			\label{noniid}
		\end{minipage}%
	\end{center}
\end{figure} 

\subsection{Impact of the Fraction of Corrupted Clients}

To demonstrate the effect of the fraction of corrupted clients to the robustness of \arfl, we run an experiment on CIFAR-10 with different number of corrupted clients and apply \textit{shuffling} to corrupt the clients. As shown in Fig.~\ref{maximum}, the performance decreases with the fraction of corrupted clients. Noticeably, the accuracy does not significantly change until more than 60\% clients are corrupted. Similar phenomenon has been shown by Li et al. in a previous work~\citep{li2020ditto}, which achieves robustness by personalized federated training. However, different from the solution in~\citep{li2020ditto}, our \arfl only learns a single global model, without the need of extra training overhead for personalized models.

\subsection{Impact of Data Heterogeneity}
In order to study the impact of non-i.i.d. data, i.e., the extent of data heterogeneity among clients, we simulate a non-i.i.d. partition of CIFAR-10, for which the number of data points and class proportions are unbalanced, and 30\% of the clients are corrupted with \textit{shuffling}. Following prior arts~\citep{hsu2019measuring,lin2020ensemble}, we model non-i.i.d. data distributions by using a Dirichlet distribution $\bm{P}_k \sim Dir_{_N}(\alpha)$ and by allocating a $\bm{P}_{k, i}$ proportion of the training instances of class $k$ to local client $i$,  in which a smaller $\alpha$ indicates stronger non-i.i.d. data distributions. As shown in Fig.~\ref{noniid}, when the data is less heterogeneous (e.g., $\alpha = 10.0$), the learning curve increases more steadily,  and therefore results in higher accuracy at the end of the training. On the contrary, the learning curve fluctuates over the training rounds when $\alpha$ is small. One explanation for this is that it becomes hard to induce a consensus model for the benign clients if their data distributions are significantly different, thus the corrupted clients can damage the global model more easily as they are not down-weighted properly.

\subsection{Tuning $\lambda$}
\label{sec_lambda}


\begin{figure}
	\begin{center}
	\begin{subfigure}[b]{0.5\linewidth}
         \centering
         \includegraphics[width=\textwidth]{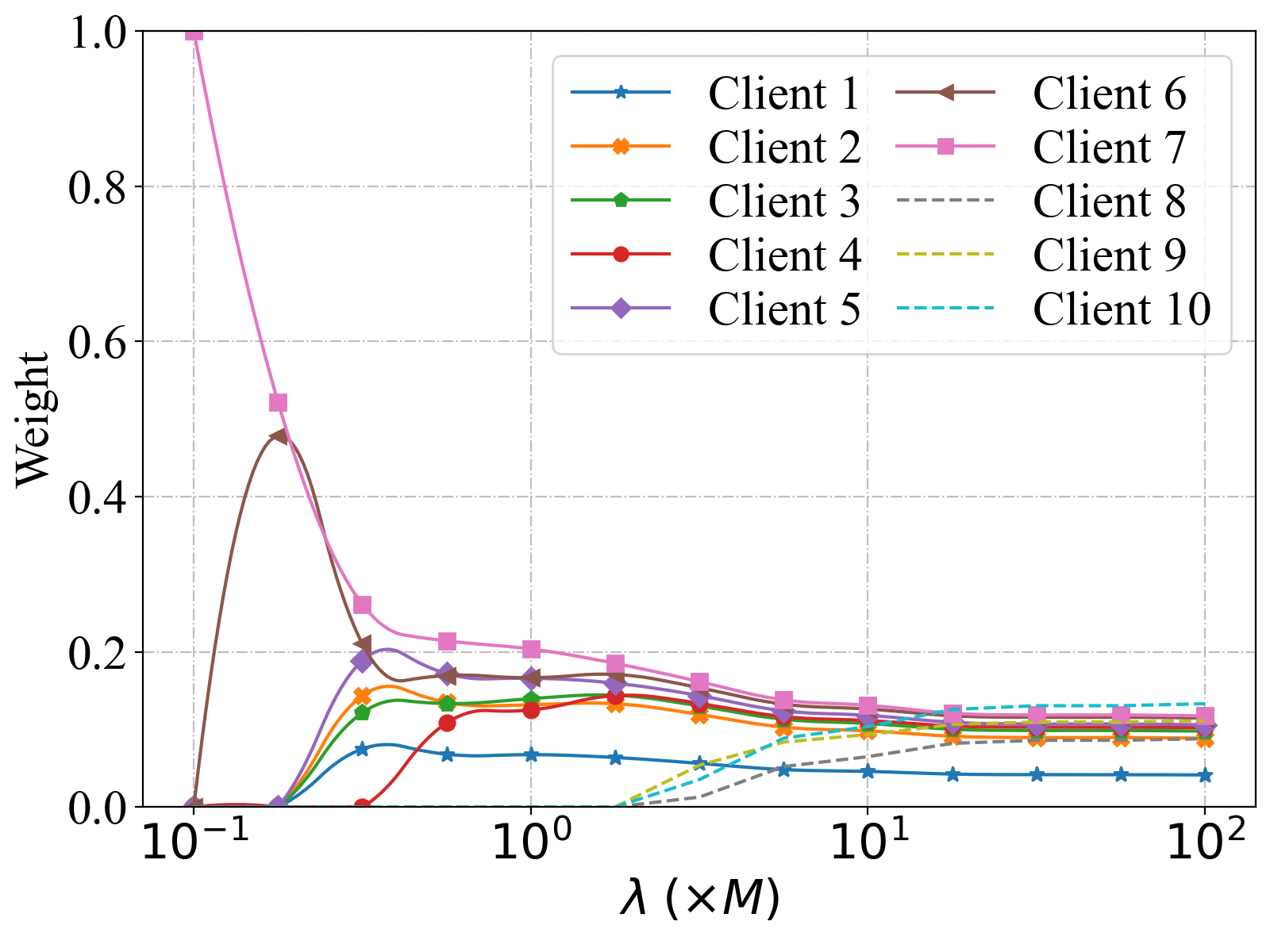}
         \caption{Weights v.s. $\lambda$}
     \end{subfigure}%
     \begin{subfigure}[b]{0.5\linewidth}
         \centering
         \includegraphics[width=\textwidth]{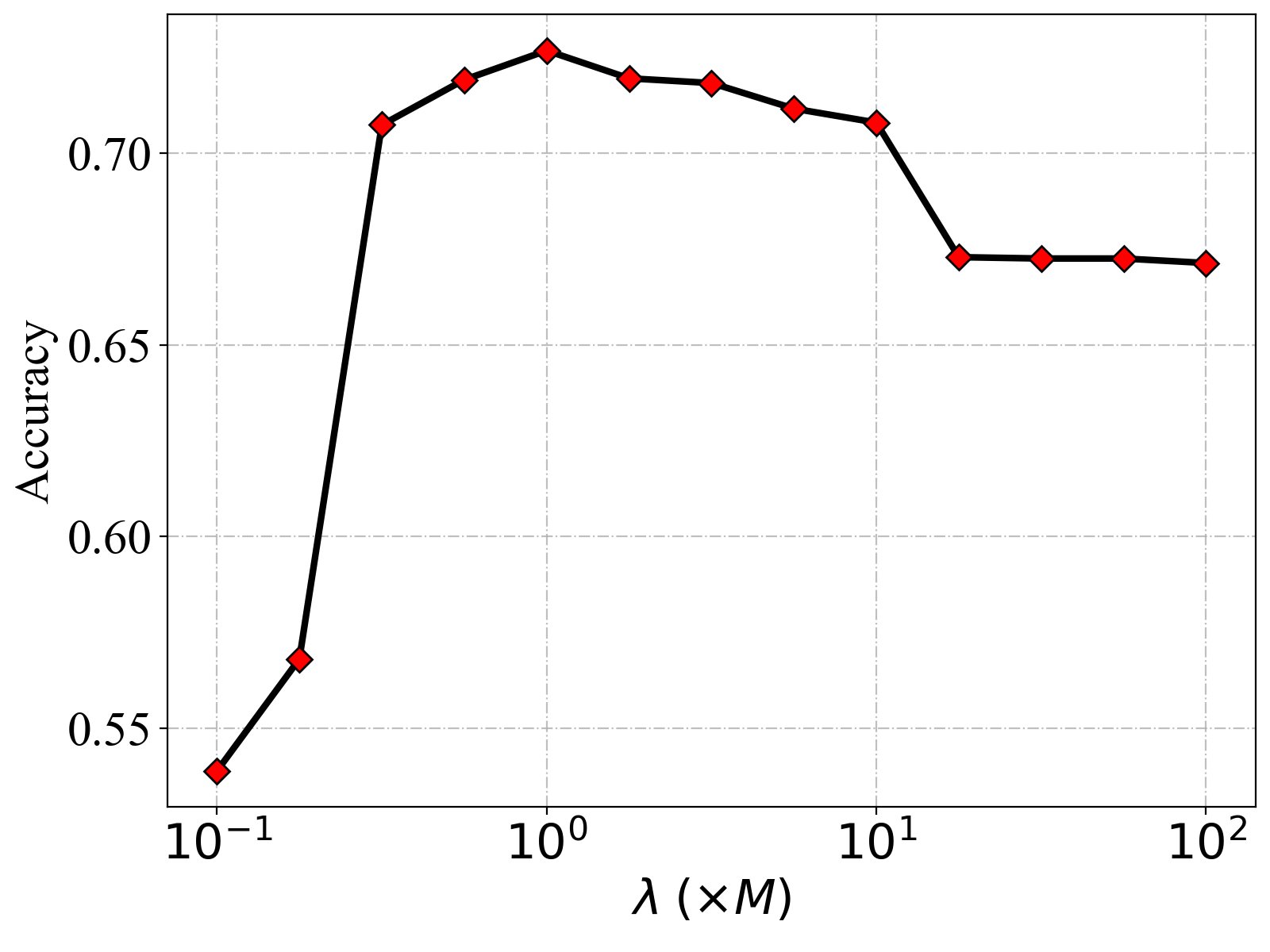}
         \caption{Accuracy v.s. $\lambda$}
     \end{subfigure}
		\caption{ Effects of $\lambda$ on weights and model accuracy on the CIFAR-10 dataset. The corrupted clients (dashed lines) are  zero-weighted when $\lambda \le 2.5 \times M$, and the accuracy reaches its peak when $\lambda = 1 \times M$.}
		\label{fig:lambda}
	\end{center}
\end{figure} 

As mentioned before, $\lambda$ in the objective should be tuned in order to provide a trade-off between robustness and average accuracy. Here we conduct further experiments on the CIFAR-10 dataset to study the effects of  $\lambda$. 
The dataset is partitioned into $N=10$ clients by sampling $\bm{P}_k \sim Dir_{_N}(0.5)$ and allocating a $\bm{P}_{k, i}$ proportion of the training instances of class $k$ to local client $i$, where three clients are corrupted by shuffling their labels randomly (Client 8-10). We use the original test set in CIFAR-10 as the global test set. We train the model for 1000 rounds where each client runs one epoch of SGD on their training set before each aggregation, where $\lambda$ is set in the range of $[10^{-1} \times M, 10^{2} \times M]$. All the other settings are the same as in Sec.~\ref{sess_setup}.

Fig.~\ref{fig:lambda}(a) shows the optimized weights as a function of $\lambda$ on the CIFAR-10 dataset. It is readily apparent that $\bm{\alpha}$ has only one non-zero element (Client 7) for small $\lambda$ and all elements of $\bm{\alpha}$ come to certain non-zero values for large $\lambda$. In between these two extremes, we obtain sparse solutions of $\bm{\alpha}$ in which only a part of elements have non-zero values.
It is also noticeable that all the corrupted clients (dashed lines) are zero-weighted when $\lambda \le 2.5 \times M$, which means that they make no contribution when the server aggregates the updated local models.

In Fig.~\ref{fig:lambda}(b), we plot the model accuracy as a function of $\lambda$, showing that the model achieves relatively low 
accuracy for small $\lambda$, which demonstrates that extremely sparse weights are not favorable under this non-i.i.d. data setting.
The reason behind this is that the model finally fits only one local dataset without considering data from other clients. The accuracy increases as more benign clients are upweighted and contribute their local updates to the global model. The optimal value for $\lambda$ is $1.0 \times M$ in this experiment, where all the benign clients have non-zero weights, while all the corrupted clients have zero weights. However, as $\lambda$ further increases, the global model is harmed by the corrupted clients as they gain adequate weights, which leads to lower accuracy.

\subsection{Auto-weighting Analysis}

\begin{figure}
	\begin{center}
	\begin{minipage}{.48\textwidth}
        \centering
        \includegraphics[width=\linewidth]{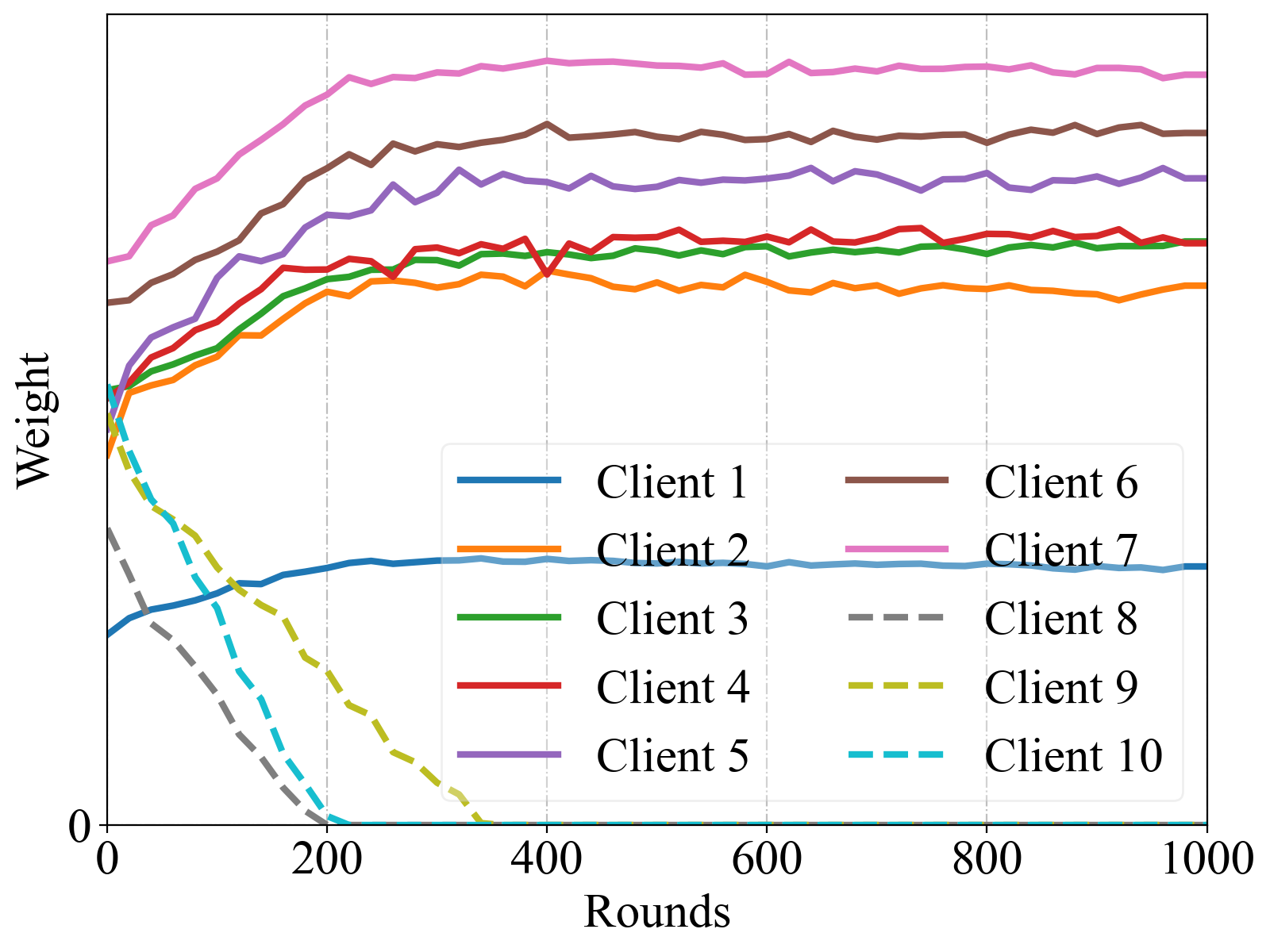}
		\caption{Visualization of the weights v.s. round number of \ARFL. All corrupted clients (dashed lines) are  zero-weighted after 350 rounds of training.}
		\label{autoweight}
    \end{minipage}%
    \hspace{0.3cm}
    \begin{minipage}{.48\textwidth}
        \centering
        \includegraphics[width=\textwidth]{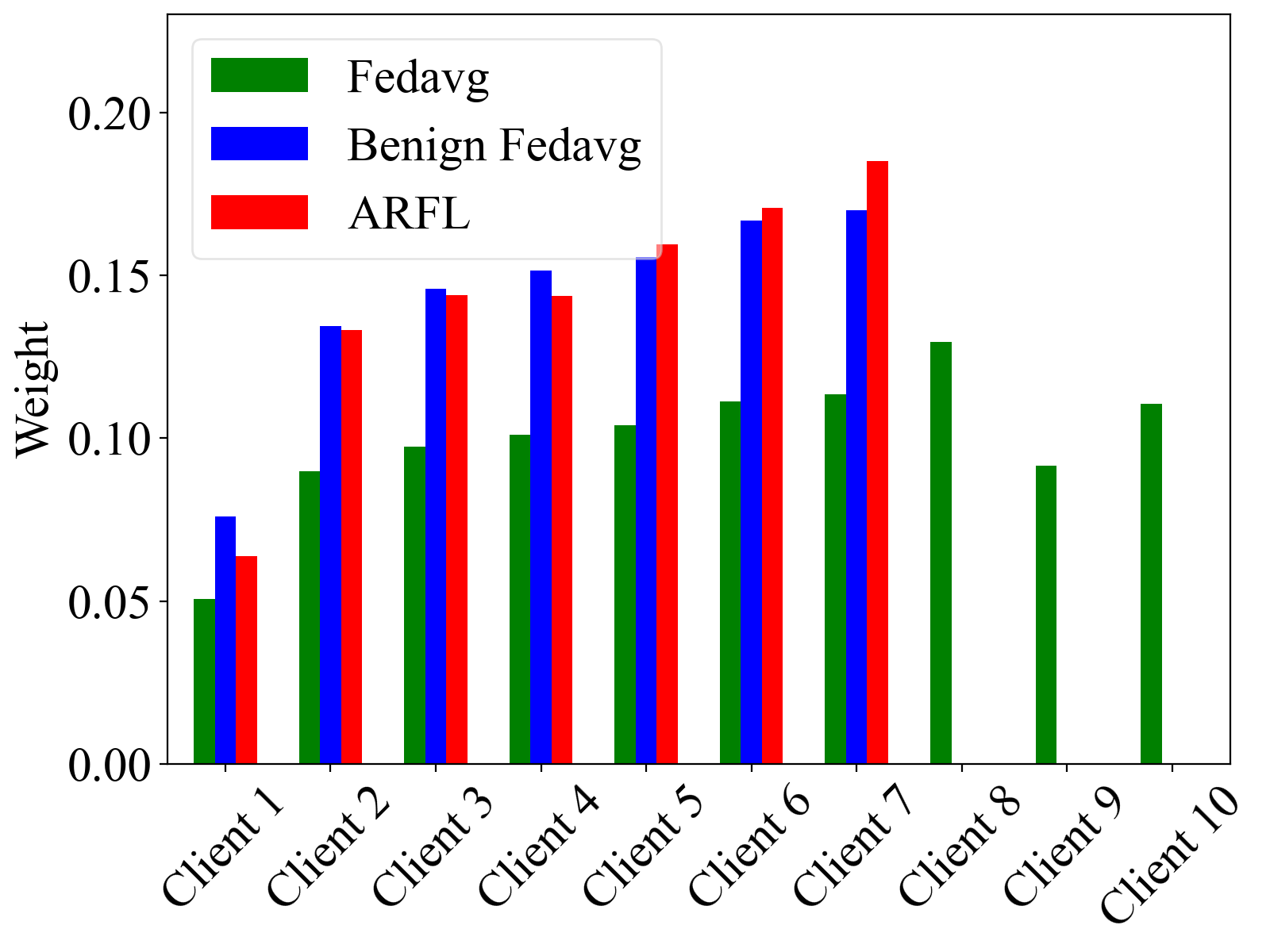}
         \caption{Weights comparison between ARFL and FedAvg.
ARFL’s weights distributed similar to those of FedAvg with
only benign clients}
		\label{weights}
    \end{minipage}%
	\end{center}
\end{figure} 

In this part, we investigate the auto-weighting process during training in \ARFL when $\lambda = 1.0 \times M$, where we keep the setup the same as the previous above. As shown in Fig.~\ref{autoweight}, our approach successfully downweights all the three corrupted clients, for which the weights become zero after 350 rounds of training. 

We next compare the reallocated weights with the standard \fedavg, where the weights are fixed as $\alpha_i = \frac{m_i}{M}$, and \fedavg with only benign clients (Benign \FedAvg), where $\alpha_i = \frac{m_i}{M_{\mathcal{B}}}$ for benign clients, and $\alpha_i = 0$ for corrupted clients. Fig.~\ref{weights} shows that the learned weight distribution (in red) of \ARFL is approximated to the distribution considering only benign clients (in blue). We conclude that our approach can automatically reweight the clients and approximate the mixture distribution to the benign uniform distribution during the model training process, even when the centralized server is agnostic to the local corruptions.

\section{Limitations}

The first limitation is the fairness. When a small $\lambda$ is chosen, the clients are treated unfairly, e.g., a client with  data that is difficult to learn (but not due to corruption) would carry less weight. This happens in some strongly non-i.i.d cases where the distributions among sources are inherently different.
However, we argue that robustness and fairness are two competing targets~\citep{li2020ditto}, it is technically difficult to distinguish between ``corrupted" and ``just different" (but not corrupted) data sources if the data is strongly non-i.i.d., which we leave it for future work. Nevertheless, we show in our experiments (Sec.~\ref{experiments}) that our approach performs well in some general non-i.i.d. settings. 

Another limitation is on the assumption of corruption strategies. Colluding corruptions among clients, e.g., multiple clients are trying to bias their data distribution to the same target, may bring risk to the robustness of the objective. In addition, our approach is not designed to handle malicious attacks on the model parameters, i.e., only data corruption scenarios are considered. Our work is based on the assumption that all clients are honest during local training and communication steps, which means that the training loss should reflect the real empirical loss of the global model on the local data. If some corrupted clients cheat the server by giving extremely small fake losses, then they will be allocated with high weights and dominate the training accordingly. 

\section{Conclusions and Future work}
In this article, we proposed Auto-weighted Robust Federated Learning (\arfl), a novel approach that automatically re-weights the local updates to lower the contribution of corrupted clients who provide low-quality updates to the global model. Experimental results on benchmark datasets corroborate the competitive performance of \arfl compared to the state-of-the-art methods under different data corruption scenarios. Our future work will focus on robust aggregation without the losses provided by clients since evaluating training loss from local clients could also be a potential overhead. Also, we will explore the possibility of further improving the robustness of federated learning and coping with higher fraction of corrupted clients. Furthermore, we plan to extend \arfl to a more general federated learning approach and make it robust against both model poisoning and data corruption.


\bibliographystyle{ACM-Reference-Format}
\bibliography{reference}

\input{appendix}

%% file: related_work.tex
\section{Related work}

The concept of federated learning has been proposed for collaboratively learning a model without collecting users' data~\citep{mcmahan2017communication,li2020ditto,konevcny2016federateda, konevcny2016federatedb}. The research work on federated learning can be divided into three categories, i.e., horizontal federated learning, vertical federated learning, and federated transfer learning, based on the distribution characteristics of the data. Due to space limits, we refer to Yang et al.~\citep{yang2019federated} for detailed explanations. In this paper we focus on horizontal federated learning where datasets in all clients share the same feature space but different samples. In 2017, one of the most famous horizontal federated learning framework, Federated Averaging (\fedavg) has been proposed to update global parameters with a weighted average of the model parameters sent by a subset of clients after several local iterations~\citep{mcmahan2017communication}.  

Due to the nature of data decentralization and the requirement of collaborative training from multiple clients, federated learning is vulnerable to malicious corruption of training data from remote clients~\mbox{\citep{lyu2020threats}}. Notably, it has been shown that traditional federated learning approaches (e.g., \fedavg) are fragile in the presence of malicious or corrupted clients~\mbox{\citep{fang2020local,so2020byzantine}}. To mitigate the impact of malicious clients and improve the robustness of federated training, several  solutions have been proposed in the literature~\mbox{\citep{so2020byzantine,pillutla2019robust,pmlrv139kairouz21a,li2020ditto}}.  Among these robust approaches, robust statistical estimations have received much attention in particular. Typical estimation rules include Geometric Median (GM) ~\citep{chen2017distributed,pillutla2019robust}, trimmed mean~\citep{yin2018byzantine}, and Krum~\citep{blanchard2017machine}. For instance, in 2017, \citet{chen2017distributed} proposed to apply GM  as a gradient aggregation protocol in robust distributed learning and showed that it can tolerate up to half malicious clients while estimating the underlying true gradients. In 2019, \citet{pillutla2019robust} used GM to aggregate parameters in a robust FL solution. However, the robustness of the traditional estimators are limited as they rely on the assumption that data are i.i.d. and balanced among the clients, i,e, they distribute identically and have the same (or a similar) number of training data points. Hence, those approaches can be inefficient when some of the clients have significantly more data than others. 
In addition, \citet{li2020ditto} proposed a multi-task learning objective called Ditto, for federated learning that provides robustness via personalization in 2021. The optimization of Ditto, however, requires extra training overhead for personalized models.

Others~\citep{li2020learning,han2020robust,sattler2020byzantine,cao2021fltrust}, couple the process of teaching and learning based on a few trusted instances to produce a robust prediction model. For example, in 2020, \citet{sattler2020byzantine} proposed to separate the client population into different groups (e.g., benign and corrupted groups) based on the pairwise cosine similarities between their parameter updates. Also in 2020, \citet{li2020learning} suggested allowing the server to learn to detect and remove the malicious model updates using an encoder-decoder based detection model. These approaches require some trusted clients or samples to guide the learning or detect the updates from corrupted clients. In 2021, \citet{cao2021fltrust} presented a Byzantine-robust method that allows the service provider itself collect a clean small training datase  for the learning task and maintain a model based on it to bootstrap trust. Unfortunately, the credibility of these trusted clients and samples are usually not guaranteed since the data is isolated and stored locally. Thus, the server is not aware of these corruption behaviors and does not have the ability to measure the quality of data at the sources due to privacy and communication constraints. Different from previous studies, we propose a robust approach that can learn both the global model and the weights of clients automatically from a mix of reliable and unreliable clients, without the need of any pre-verified trusted instances.

%% file: preliminaries.tex
\section{Preliminaries and Motivation}

\subsection{Federated Learning}
In federated learning tasks, a general assumption is that the target distribution for which the centralized model is learned  is a weighted mixture of distributions associated with all clients (or data sources), that is, if we denote by ${\mathcal{D}}_{i}$ the distribution associated with the $i$-th client, the centralized model is trained to minimize the risk with respect to $\mathcal{D}_{\bm{\alpha}} = \sum_{i=1}^N \alpha_i {\mathcal{D}}_{i}$, where $N$ is the total number of clients, $\bm{\alpha} = (\alpha_1 ,..., \alpha_N)^{\top}$ is the vector of source-specific weights. We also have $\bm{\alpha} \in \mathbb{R}_+^n$ and $\bm{1}^\top\bm{\alpha}=1$~\cite{mohri2019agnostic, hamer2020fedboost,li2020ditto}.

Let $\loss_{h}(\bm{z})$ be the loss function that captures the error of a predictor $h \in \mathcal{H}$ (where $\mathcal{H}$ is the hypothesis class) on the training data $\bm{z} = (\bm{x}, y)$ (where $(\bm{x}, y)$ is the input and output pair), and $\mathcal{L}_{\mathcal{D}_{\bm{\alpha}}}(h)$ be the risk of a predictor $h$ on the mixture data distribution $\mathcal{D}_{\bm{\alpha}}$, we have:
\begin{equation} 
 \mathcal{L}_{\mathcal{D}_{\bm{\alpha}}}(h) = 
 \sum_{i=1}^N \alpha_i {\mathcal{L}}_{{i}}(h)  = 
 \sum_{i=1}^N \alpha_i \expect_{\bm{z}\sim \mathcal{D}_i} (\loss_{h}(\bm{z})),
 \label{exp_risk}
\end{equation} where ${\risk}_{i}(h) = \expect_{\bm{z}\sim \mathcal{D}_i} (\loss_{h}(\bm{z}))$ is the expected loss of a predictor $h$ on the data distribution $\mathcal{D}_i$ of the $i$-th client.

Most prior work in federated learning has assumed that all samples are uniformly weighted, where the underlying assumption is that the target distribution is $\uniform =  \sum_{i = 1}^N \frac{m_i}{M} \mathcal{D}_i $, where $m_i$ is the number of samples from client $i$ and $M = \sum_{i=1}^N m_i$.
Thus the risk becomes:
\begin{equation} 
  \mathcal{L}_{\uniform}(h) =  
 \sum_{i=1}^N \frac{m_i}{M} \expect_{\bm{z}\sim \mathcal{D}_i} (\loss_{h}(\bm{z})).
 \label{exp_risk_m}
\end{equation}

In practice, the goal is to minimize a traditional empirical risk ${\hat{\risk}}_{\uniform}(h)$ as follows:
\begin{equation} 
 {\hat{\risk}_{\uniform}}(h) = 
 \sum_{i=1}^N \frac{m_i}{M} \frac{1}{m_i} \sum_{j=1}^{m_i} \loss_{h}(\bm{z}_{i,j}),
 \label{emp_risk_m}
\end{equation} 
which can be minimized by sampling a subset of clients randomly at each round, then running an optimizer such as stochastic gradient descent (SGD) for a variable number of iterations locally on each client. These local updating methods enable flexible and efficient communication compared to traditional mini-batch methods, which would simply calculate a subset of the gradients~\cite{wang2019cooperative,stich2018local,yu2019parallel}. One of the most well-known methods to minimize Eq.~(\ref{emp_risk_m}) in non-convex settings is \fedavg~\cite{mcmahan2017communication}, which runs simply by letting each selected client apply a fixed number of epochs of SGD locally and then averaging the resulted local models.

\subsection{Threat Model: Data Corruption}
Guerraoui et al. has shown that a few  clients with corrupted data can lead to inaccurate models~\cite{mhamdi2018hidden}. The problem stems from a mismatch between the target distribution and $\uniform$. That is, in corruption scenarios, the target distribution may in general be quite different from $\uniform$, since $\uniform$ includes some corrupted components. We expect that the data distributions are more similar among benign clients compared with the corrupted ones even when the data is non-i.i.d. More specifically, we model the target distribution with corrupted clients as, 
\begin{align}
\mathcal{D}_{\bm{\alpha}} &= \sum_{i = 1}^N \alpha_i \mathcal{D}_i = \sum_{i = 1}^N \alpha_i (\eta_i \mathcal{D}_{i,b} + (1 - \eta_i) \mathcal{D}_{i, c}),  
\end{align}
where $\eta_i \in \{0, 1\}$ denotes whether the local data distribution $\mathcal{D}_i$ is a benign distribution  $\mathcal{D}_{i,b}$ (when $\eta_i = 1$) or a corrupted distribution $\mathcal{D}_{i, c}$ (when $\eta_i = 0$). Ideally, the corruption can be usually measured by the difference between $\mathcal{D}_{i,b}$ and $\mathcal{D}_{i, c}$ using statistical distance (e.g., the Kullback-Leibler divergence), if both distributions are known~\cite{konstantinov2019robust}. In practice, data corruption can be achieved by modifying features or labels of the local training dataset.

When $\sum_{i=1}^N {\eta_i} = N$, all components of the mixture distribution are benign (i.e. $\mathcal{D}_i = \mathcal{D}_{i,b}$). Assuming that the target distribution is uniform $\uniform$, minimizing Eq.~(\ref{emp_risk_m}) can lead to an accurate global model. However, when there are corrupted data sources (i.e. $\sum_{i=1}^N {\eta_i} < N$), the mixture of the distributions will include some corrupted components $\mathcal{D}_{i,c}$. In this case, optimizing the empirical risk with respect to $\uniform$ will not lead to an accurate global model. 
{
We make the following assumptions regarding the adversaries: 
\begin{enumerate}
    \item~Each adversary controls exactly one non-colluding and corrupted client. The data distributions of any corrupted clients are independent of each other. Since all clients do not collude, the effect of malicious updates from each adversary to the global model is limited~\cite{bhagoji2019analyzing}.
    
    \item~A corrupted client has a higher loss with respect to the best predictor under the uniformly weighted assumption in Eq.~(\ref{exp_risk_m}), i.e., we have
        \begin{equation}
            \hat{\risk}_{{D}_{i,c}}(\hstar) > \hat{\risk}_{{D}_{i,b}}(\hstar),
        \end{equation} where ${h}^{\ast}$ is the optimal global predictor that minimizes the empirical risk. Base on this, we will identify the corrupted clients according to their empirical losses as discussed in the next section.
\end{enumerate}

}


\subsection{Introductory Example}

\begin{figure*}[t]
     \centering
     \includegraphics[width=\textwidth]{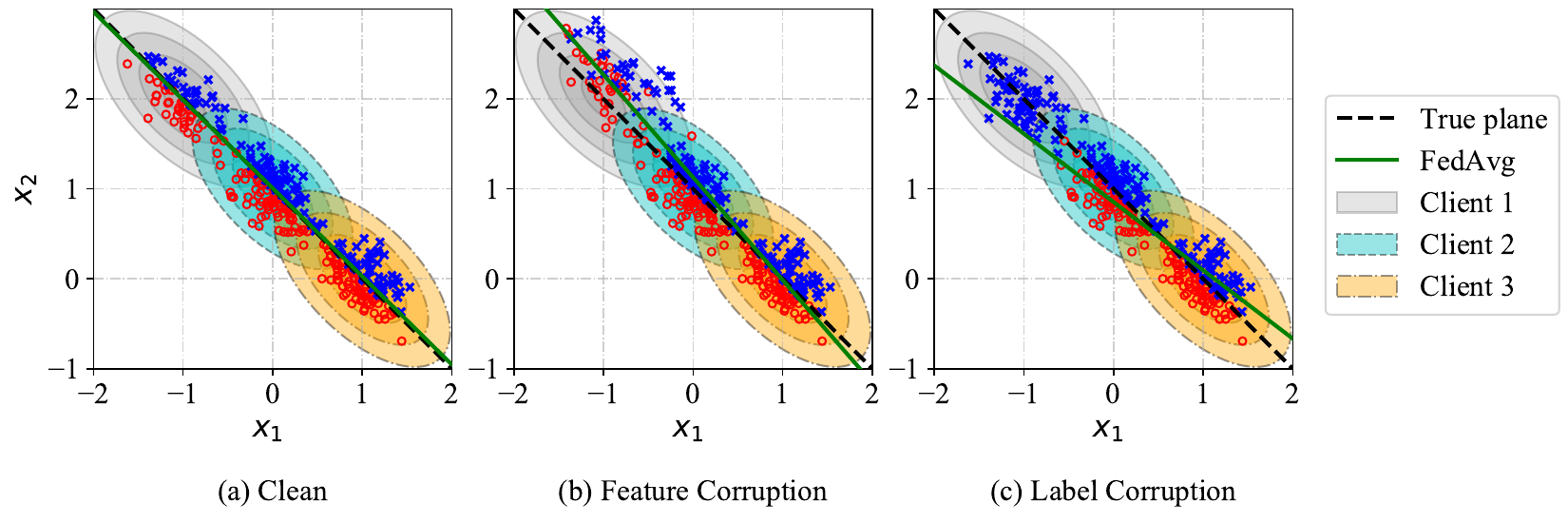}
     \caption{Illustration of federated learning using the standard \FedAvg with potential corruptions in local datasets, where red circles represent data of class 0 and blue crosses represent class 1. (a) All samples are clean. (b) The features in Client 1 are shifted by random noise. (c) The labels in Client 1 are forced to class 1.}
     \label{fig_example}
\end{figure*} 
As an example, we train a binary classification model from three clients, where the data is originally generated from three distributions with linearly separable classes. A Logistic Regression (LR) model is learned using the standard \FedAvg approach. As shown in Fig.~\ref{fig_example}(a), the learned separating plane is close to the true plane when the all datasets are clean. When Client 1 is corrupted on either feature $\bm{x}$ (Fig.~\ref{fig_example}(b)) or label $y$ (Fig.~\ref{fig_example}(c)), the learned plane is driven away from the true one.

One possible solution is to exclude all the corrupted components of the mixture distribution and seek for a new target distribution considering only benign clients, i.e. $\uniform_b = \sum_{i \in \mathcal{B}} \frac{m_i}{M_{\mathcal{B}}} \mathcal{D}_i$, where $\mathcal{B}$ is the set of all benign clients and $M_{\mathcal{B}} = \sum_{i \in \mathcal{B}}m_i$. The challenge with this approach is that the centralized server is agnostic to the corruptions on the local clients and hence it is impossible to measure the data quality directly. 

%% file: appendix.tex
\appendix

\section{Proof of Theorem \ref{thm:main_bound}}
\label{proof1}
\mainthm*



\begin{proof}
Write:
\begin{equation}
\label{eqn:equation_i_need}
    \mathcal{L}_{\mathcal{D}_{\bm{\alpha}}}\left(h\right) \leq \hat{\mathcal{L}}_{\mathcal{D}_{\bm{\alpha}}}\left(h\right) + \sup_{f\in \mathcal{H}}\left(\mathcal{L}_{\mathcal{D}_{\bm{\alpha}}}\left(f\right)- \hat{\mathcal{L}}_{\mathcal{D}_{\bm{\alpha}}}\left(f\right)\right)
\end{equation}
To link the second term to its expectation, we prove the following:
\begin{lemma}
\label{lem:mcdiar_condition}
Define the function $\phi:\left(\mathcal{X}\times\mathcal{Y}\right)^m \rightarrow \mathbb{R}$ by:
\begin{small}
$$ \phi\left(\{x_{1,1}, y_{1,1}\}, \ldots, \{x_{N, m_N}, y_{N, m_N}\}\right) = \sup_{f\in \mathcal{H}}\left(\mathcal{L}_{\mathcal{D}_{\bm{\alpha}}}\left(f\right)- \hat{\mathcal{L}}_{\mathcal{D}_{\bm{\alpha}}}\left(f\right)\right).$$
\end{small}
Denote for brevity $z_{i,j} = \{x_{i,j}, y_{i,j}\}$. Then, for any $i \in \{1, 2, \ldots, N\}, j \in \{1, 2, \ldots, m_i\}$:
\begin{equation}
\begin{split}
    \sup_{z_{1,1}, \ldots, z_{N, m_N}, z_{i,j}^{'}} |\phi\left(z_{1,1}, \ldots, z_{i,j}, \ldots, z_{N, m_N}\right)  - \phi\left(z_{1,1}, \ldots, z_{i,j}^{'}, \ldots, z_{N, m_N}\right)| \leq \frac{\alpha_i}{m_i}\mathcal{M}
\end{split}
\end{equation}
\end{lemma}
\begin{proof}
Fix any $i, j$ and any $z_{1,1}, \ldots, z_{N, m_N}, z_{i,j}^{'}$. Denote the $\alpha$-weighted empirical average of the loss with respect to the sample $z_{1,1}, \ldots, z_{i,j}^{'}, \ldots, z_{N, m_N}$ by $\mathcal{L}_{\mathcal{D}_{\bm{\alpha}}}^{'}$. Then we have that:
\begin{align*}
    |\phi\left(\ldots, z_{i,j}, \ldots\right)  - \phi (\ldots, z_{i,j}^{'}, \ldots )| 
    & =  |\sup_{f\in\mathcal{H}}\left(\mathcal{L}_{\mathcal{D}_{\bm{\alpha}}}\left(f\right) - \hat{\mathcal{L}}_{\mathcal{D}_{\bm{\alpha}}}\left(f\right)\right)  - \sup_{f\in\mathcal{H}} (\mathcal{L}_{\mathcal{D}_{\bm{\alpha}}}\left(f\right) - \hat{\mathcal{L}}_{\mathcal{D}_{\bm{\alpha}}}^{'}\left(f\right) )| \\ 
    & \leq |\sup_{f\in\mathcal{H}}(\hat{\mathcal{L}}^{'}_{\mathcal{D}_{\bm{\alpha}}}\left(f\right) - \hat{\mathcal{L}}_{\mathcal{D}_{\bm{\alpha}}}\left(f\right))| \\ 
    & = \frac{\alpha_i}{m_i}|\sup_{f\in\mathcal{H}}\left(\loss_f(z'_{i,j}) - \loss_f(z_{i,j})\right)| \\
    & \leq \frac{\alpha_i}{m_i}\mathcal{M}
\end{align*}
Note: the inequality we used above holds for bounded functions inside the supremum.
\end{proof}
\noindent 
Let $S$ denote a random sample of size $m$ drawn from a distribution as the one generating out data (i.e. $m_i$ samples from $\mathcal{D}_i$ for each $i$). Now, using Lemma \ref{lem:mcdiar_condition}, McDiarmid's inequality gives:
\begin{equation*}
\begin{split}
    \mathbb{P}\left(\phi(S) - \mathbb{E}(\phi(S)) \geq t\right) & \leq \exp\left(-\frac{2t^2}{\sum_{i=1}^N\sum_{j=1}^{m_i}\frac{\alpha_i^2}{m_i^2}\mathcal{M}^2} \right) \\ & = \exp\left(-\frac{2t^2}{\mathcal{M}^2\sum_{i=1}^N \frac{\alpha_i^2}{m_i}}\right)
\end{split}
\end{equation*}
For any $\delta > 0$, setting the right-hand side above to be $\delta/4$ and using (\ref{eqn:equation_i_need}), we obtain that with probability at least $1-\delta/4$:
\begin{equation}
\begin{split}
    \mathcal{L}_{\mathcal{D}_{\bm{\alpha}}}\left(h\right) \leq \hat{\mathcal{L}}_{\mathcal{D}_{\bm{\alpha}}}\left(h\right) & + \mathbb{E}_S\left(\sup_{f\in\mathcal{H}}\left(\mathcal{L}_{\mathcal{D}_{\bm{\alpha}}} (f) - \hat{\mathcal{L}}_{\mathcal{D}_{\bm{\alpha}}}(f)\right)\right)  + \sqrt{\frac{\log\left(\frac{4}{\delta}\right)\mathcal{M}^2}{2}}\sqrt{\sum_{i=1}^N\frac{\alpha_i^2}{m_i}}
\end{split}
\end{equation}
To deal with the expected loss inside the second term, introduce a ghost sample (denoted by $S'$), drawn from the same distributions as our original sample (denoted by $S$). Denoting the weighted empirical loss with respect to the ghost sample by $\mathcal{L}_{\mathcal{D}_{\bm{\alpha}}}^{'}$, $\beta_i = m_i/m$ for all $i$, and using the convexity of the supremum, we obtain:
\begin{equation*}
\begin{split}
    \mathbb{E}_S  \left(\sup_{f\in\mathcal{H}}\left(\mathcal{L}_{\mathcal{D}_{\bm{\alpha}}} (f) - \hat{\mathcal{L}}_{\mathcal{D}_{\bm{\alpha}}}(f)\right)\right) & = \mathbb{E}_{S}\left(\sup_{f\in\mathcal{H}}\left(\mathbb{E}_{S'}\left(\hat{\mathcal{L}}_{\mathcal{D}_{\bm{\alpha}}}^{'}(f)\right) - \hat{\mathcal{L}}_{\mathcal{D}_{\bm{\alpha}}}(f)\right)\right) \\ & \leq \mathbb{E}_{S, S'} \left(\sup_{f\in\mathcal{H}}\left(\hat{\mathcal{L}}_{\mathcal{D}_{\bm{\alpha}}}^{'}(f) - \hat{\mathcal{L}}_{\mathcal{D}_{\bm{\alpha}}}(f) \right)\right) \\ 
    & = \mathbb{E}_{S, S'}\left(\sup_{f\in\mathcal{H}}\left(\frac{1}{m}\sum_{i=1}^N\sum_{j=1}^{m_i}\frac{\alpha_i}{\beta_i}\left(\loss_f(z'_{i,j}) \right. \right. \right.   \left. \left. \left.  - \loss_f(z_{i,j}) \vphantom{L^{'}}\right) \vphantom{\frac{1}{m}\sum_{i=1}^N}\right)\right)
\end{split}
\end{equation*}

Introducing $m$ independent Rademacher random variables and noting that $\left(\loss_f(z') - \loss_f(z)\right)$ and $\sigma\left(\loss_f(z') - \loss_f(z)\right)$ have the same distribution, as long as $\bm{z}$ and $\bm{z}'$ have the same distribution:
\begin{equation*}
\begin{split}
    \mathbb{E}_S  \left(\sup_{f\in\mathcal{H}}\left(\mathcal{L}_{\mathcal{D}_{\bm{\alpha}}} (f) - \hat{\mathcal{L}}_{\mathcal{D}_{\bm{\alpha}}}(f)\right)\right)  & \leq  \mathbb{E}_{S, S', \sigma}\left(\sup_{f\in\mathcal{H}}\left(\frac{1}{m}\sum_{i=1}^N\sum_{j=1}^{m_i}\frac{\alpha_i}{\beta_i}\sigma_{i,j}\left(\loss_f(z_{i,j})^{'}) \right. \right. \right. \left. \left. \left. - \loss_f(z_{i,j}) \vphantom{L^{'}}\right) \vphantom{\frac{1}{m}\sum_{i=1}^N}\right)\right) \\ 
    & \leq \mathbb{E}_{S^{'}, \sigma}\left(\sup_{f\in\mathcal{H}}\left(\frac{1}{m}\sum_{i=1}^N\sum_{j=1}^{m_i}\frac{\alpha_{i}}{\beta_{i}}\sigma_{i,j}\loss_f(z_{i,j})\right)\right) \\ & +  \mathbb{E}_{S, \sigma}\left(\sup_{f\in\mathcal{H}}\left(\frac{1}{m}\sum_{i=1}^N\sum_{j=1}^{m_i}\frac{\alpha_{i}}{\beta_{i}}\left(-\sigma_{i,j}\right)\loss_f(z_{i,j})\right)\right) \\ 
    & = 2\mathbb{E}_{S, \sigma}\left(\sup_{f\in\mathcal{H}}\left(\frac{1}{m}\sum_{i=1}^N\sum_{j=1}^{m_i}\frac{\alpha_{i}}{\beta_{i}}\sigma_{i,j}\loss_f(z_{i,j})\right)\right).
\end{split}
\end{equation*}
We can now link the last term to the empirical analog of the Rademacher complexity, by using the McDiarmid Inequality (with an observation similar to Lemma 1). Putting this together, we obtain that for any $\delta > 0$ with probability at least $1 - \delta/2$:
\begin{equation}
\begin{split}
    \mathcal{L}_{\mathcal{D}_{\bm{\alpha}}}\left(h\right) & \leq \hat{\mathcal{L}}_{\mathcal{D}_{\bm{\alpha}}} \left(h\right) 
     + 2\mathbb{E}_{\sigma}\left(\sup_{f\in\mathcal{H}}\left(\frac{1}{m}\sum_{i=1}^N\sum_{j=1}^{m_i}\frac{\alpha_{i}}{\beta_{i}}\sigma_{i,j}\loss_f(z_{i,j})\right)\right) 
     + 3 \sqrt{\frac{\log\left(\frac{4}{\delta}\right)M^2}{2}}\sqrt{\sum_{i=1}^N\frac{\alpha_i^2}{m_i}}
\end{split}
\end{equation}

Finally, note that:
\begin{align*}
    \mathbb{E}_{\sigma} \left(\sup_{f\in\mathcal{H}}\left(\frac{1}{m}\sum_{i=1}^N\sum_{j=1}^{m_i}\frac{\alpha_{i}}{\beta_{i}}\sigma_{i,j}\loss_f(z_{i,j})\right)\right) & \leq \mathbb{E}_{\sigma}\left(\sum_{i=1}^{N}\alpha_i\sup_{f\in\mathcal{H}}\left(\frac{1}{m_i}\sum_{j=1}^{m_i}\sigma_{i,j}\loss_f(z_{i,j})\right)\right) \\ 
    & = \sum_{i=1}^N \alpha_i \mathbb{E}_{\sigma}\left(\sup_{f\in\mathcal{H}}\left(\frac{1}{m_i}\sum_{j=1}^{m_i}\sigma_{i,j}\loss_f(z_{i,j})\right)\right) \\ 
    & = \sum_{i=1}^N \alpha_i \mathcal{R}_i \left(\mathcal{H}\right)
\end{align*}
Bounding $\hat{\mathcal{L}}_{\mathcal{D}_{\bm{\alpha}}}(h) - \mathcal{L}_{\mathcal{D}_{\bm{\alpha}}}(h)$ with the same quantity and with probability at least $1 - \delta/2$ follows by a similar argument. The result then follows by applying the union bound.
\end{proof}

\section{Proof of Theorem \ref{thm:alpha}}
\label{proof2}
\alphathm*

\begin{proof} 

The Lagrangian function of Eq. (\ref{eq:fl:main}) is
\begin{equation}
\mathbb{L} = \bm{\alpha}^\top {\hat{\risk}}(\bm{w}) + \frac{\lambda}{2}  || \bm{\alpha}^{\top} \bm{m}^{\circ - \frac{1}{2}} ||^2_2 - \bm{\alpha}^{\top} \bm{\beta} - \eta(\bm{\alpha}^{\top} \bm{1} - 1),
\end{equation}
where $\hat{\risk}(\bm{w}) = [\hat{\risk}_1(\bm{w}),\hat{\risk}_2(\bm{w}), ..., \hat{\risk}_N(\bm{w})]^\intercal$, $\circ$ is the Hadamard root operation, $\bm{\beta}$ and  $\eta$ are the Lagrangian multipliers.  Then the following Karush-Kuhn-Tucker (KKT) conditions hold:
\begin{align}
   \partial_{\bm{\alpha}} \mathbb{L}(\bm{\alpha}, \bm{\beta}, \eta) &= 0 \label{main_diff},\\
   \bm{\alpha}^\intercal \bm{1} - 1 &= 0, \\
   \bm{\alpha} &\ge 0, \\
   \bm{\beta} &\ge 0, \\
   \alpha_i \beta_i &= 0 , \forall i = 1, 2, ... N.
\end{align}
According to Eq.~(\ref{main_diff}), we have:
\begin{equation}
    \alpha_i = \frac{m_i(\beta_i + \eta - \hat{\mathcal{L}}_i(\bm{w}))}{\lambda}.
\end{equation}
Since $\beta_i \ge 0$, we discuss the following cases:
\begin{enumerate}
  \item When $\beta_i = 0$, we have $\alpha_i = \frac{m_i(\eta - \hat{\mathcal{L}}_i(\bm{w}))}{\lambda} \ge 0$. Note that we further have $\eta - \hat{\mathcal{L}}_i(\bm{w}) \ge 0$.
  \item When $\beta_i > 0$, from the condition $\alpha_i \beta_i = 0$, we have $\alpha_i = 0$. 
\end{enumerate}
Therefore, the optimal solution to Eq.~(\ref{eq:fl:main}) is given by:
\begin{equation}
\label{alpha}
\alpha_i(\bm{w}) = [\frac{m_i (\eta - \hat{\mathcal{L}}_i(\bm{w}))}{\lambda}]_{+},
\end{equation}
where $[\cdot]_+ = max(0, \cdot)$.
		
 We notice that $\sum_{i=1}^p \alpha_i = 1$, thus we can get:
 \begin{equation}
    \label{eta}
			\eta = \frac{\sum_{i=1}^{p} m_i \hat{\mathcal{L}}_i(\bm{w})  + \lambda}{\sum_{i=1}^{p} m_i}.
\end{equation}

According to $\eta - \hat{\mathcal{L}}_i(\bm{w}) \ge 0$, we have Eq.~(\ref{nonzero}) and Eq.~(\ref{avg}). Finally, plugging Eq.~(\ref{eta}) into Eq.~(\ref{alpha}) yields Eq.~(\ref{inner_solution}).
\end{proof}
